\documentclass[11pt]{article} % Anonymized submission
\usepackage{amsmath, amssymb, natbib, graphicx, url,color}
\usepackage{amsthm}
\usepackage[boxed]{algorithm2e}
\usepackage[T1]{fontenc}

\usepackage{bbm}
\usepackage{tikz}
\usetikzlibrary{arrows}
\usepackage{enumerate}
\usepackage{longtable}
\usepackage{float}
\usepackage{wrapfig}

\usepackage{textcomp}
\usepackage{latexsym}
\usepackage{multicol}
\usepackage{dsfont}

\usepackage{subcaption}

\DeclareMathOperator{\argmax}{argmax}
\renewcommand{\epsilon}{\varepsilon}
\renewcommand{\leq}{\leqslant}
\renewcommand{\geq}{\geqslant}
\newtheorem{assumption}{Assumption}
\title{Sparse Stochastic Bandits}

\newtheorem{lemma}{Lemma}
\newtheorem{theorem}{Theorem}

\newtheorem{definition}{Definition}
\newtheorem{remark}{Remark}

\usepackage{times}
 \author{Joon Kwon  \\
 CMAP, \'Ecole polytechnique, Universit\'e Paris--Saclay\\
\texttt{joon.kwon@ens-lyon.org}
 \and
 Vianney Perchet \\
 CMLA, \'Ecole Normale Sup\'erieure Paris--Saclay\\ \& Criteo Research\\
 \texttt{vianney.perchet@normalesup.org} \and
 Claire Vernade\thanks{J.\ Kwon was supported by a public grant as part of the
  Investissement d'avenir project, reference ANR-11-LABX-0056-LMH. V.\ Perchet has benefitted from the support of the ANR (grant ANR-13-JS01-0004-01), of the \emph{FMJH Program Gaspard Monge in optimization and operations research} (supported in part by EDF) and from the Labex LMH. C.\ Vernade  was also partially supported by the Machine Learning for Big Data Chair at T\'el\'ecom ParisTech.}
 \\
  LTCI, T\'el\'ecom ParisTech\\
  \texttt{claire.vernade@telecom-paristech.fr}
 }

\usepackage[backgroundcolor=White,textwidth=1in]{todonotes}

\newcommand\blfootnote[1]{%
	\begingroup
	\renewcommand\thefootnote{}\footnote{#1}%
	\addtocounter{footnote}{-1}%
	\endgroup
}

\newcommand{\algosparse}{\textsc{SparseUCB}}

\newcommand{\ucb}{\textsc{UCB}}
\newcommand{\N}{\mathds{N}}

\DeclareMathOperator{\Reg}{Reg}
\begin{document}

\maketitle

\begin{abstract}
In the classical multi-armed bandit problem, $d$ arms are available to
the decision maker who pulls them sequentially in order to maximize
his cumulative reward. Guarantees can be obtained on a relative quantity called regret, which scales linearly with $d$ (or with
$\sqrt{d}$ in the minimax sense). We here consider the \emph{sparse
  case} of this classical problem in the sense that only a small number of arms, namely $s<d$, have a \emph{positive} expected reward. We are able to leverage this additional assumption to provide an algorithm whose regret scales with $s$ instead of $d$. Moreover, we prove that this algorithm is optimal by providing a matching lower bound -- at least for a wide and pertinent range of parameters that we determine -- and by evaluating its performance on simulated data.
\end{abstract}

\blfootnote{Accepted for presentation at Conference on
	Learning Theory (COLT) 2017}

% \newcommand\blfootnote[1]{%
%   \begingroup
%   \renewcommand\thefootnote{}\footnote{#1}%
%   \addtocounter{footnote}{-1}%
%   \endgroup
% }
% \footnote{Accepted for presentation at Conference on
% Learning Theory (COLT) 2017}

\section{Introduction}
\label{sec:introduction}
%!TEX root = main_colt.tex
%

We consider  the celebrated stochastic multi-armed bandit problem \cite{robbins1985some}, 
where a decision maker sequentially samples from $d\geqslant 1$ processes, also called \emph{arms}, aiming at maximizing its cumulative reward. Specifically, those arms are characterized by their distributions $\nu_1,\dots,\nu_d$ and pulling arm $i\in [d]:=\{1,...,d\}$ at time $t$ yields a reward $X_{i}(t)\sim \nu_{i}$, the sequence $(X_{i}(t))_{t\geq 1}$ being assumed to be \emph{i.i.d}.  
There are many motivations behind the study of those models, ranging from random clinical trials, to maximization of the click through rate, portfolio optimization, etc.

An algorithm (or \emph{policy}) maps anterior observations to the next arm $I(t)\in [d]$ to be pulled. The performance of a given algorithm is evaluated by its \emph{cumulative
regret} $\operatorname{Reg}(T)$ defined as the difference between the  rewards gathered by the sequence
$(I(t))_{1\leq t \leq T}$ and those that might have been obtained in expectation by always behaving optimally, that is, by pulling the arm with maximal mean $\mu_i:=\mathbb{E}_{\nu_i}[X]$ at each round:
\[
  \operatorname{Reg}(T) = T\mu_* -\sum_{t=1}^{T} X_{I(t)}(t) \quad \text{ where } \mu_* = \max_{i\in[d]} \mu_i\ .
\]
The classical multi-armed bandit problem is now well understood, and there exist algorithms minimizing the regret such that
\[\mathbb{E}\left[    \operatorname{Reg}(T) \right]\lesssim \sum_{\substack{i\in [d]\\\Delta_i>0}} \frac{\log(T)}{\Delta_i}, \quad \text{ where } \ \Delta_i = \mu_* - \mu_i\ , \]
and where the notation $\lesssim$ indicates that the inequality holds up
to some universal multiplicative constants and some additive
constants\footnote{We focus, for the sake of clarity, on the leading
  terms in $T$ with explicit dependencies in the different parameters
  of the problems.}. 
  Converse statements have also been proved, first by \cite{lai1985asymptotically} and then by \cite{burnetas1996optimal}: 
  Any \emph{consistent}
policy (i.e.\ whose regret is always less than $T^\alpha$ for all
$\alpha>0$) always have a regret larger than  $\sum_{i=1}^d \frac{\log(T)}{\Delta_i}$ (again, up to some  constants). 

When $T$ is fixed and the parameters $\mu_i$ are chosen to maximize regret, the \emph{distribution-independent bounds} are of order $\sqrt{dT}$ as first shown in \cite{cesa2006prediction}.

\medskip

The main drawback of those results is that the regret scales linearly
with the number of arms $d$,or with $\sqrt{d}$ in the minimax
analysis. Since upper and lower bounds match, this is actually
ineluctable. On the other hand, we aim at leveraging an additional
assumption to reduce that (linear) dependency in $d$ and even get rid
of it, if possible. We therefore define and investigate the
\emph{sparse bandit problem} (SPB) where the decision maker knows
\emph{a priori} that only $s$ of the $d$ arms have a
\emph{significant} mean $\mu_i$.

Specifically, we assume that exactly $s$ arms have positive means\footnote{Equivalently, we could be given a threshold $\tau$ and the exact number $s$ of arms with means strictly greater than $\tau$.}. Without loss of generality, we number the arms in
nonincreasing order and write
\[ \mu_*=\mu_1\geqslant \mu_2\geqslant \dots \geqslant \mu_s>0\geqslant \mu_{s+1}\geqslant \dots \geqslant \mu_d. \]
% To be more precise, we first consider that, without loss of generality, $\mu_1 \geqslant \mu_2\geq \ldots \geq \mu_d$. Then there exists  $s$ arms with a positive expected reward, i.e., $$\mu_1 \geq \mu_2 \geq \ldots \geq\mu_s > 0 \geq \mu_{s+1} \geq \ldots \geq \mu_d\ .$$
A key quantity will be the lowest positive mean $\mu_s$: if $\mu_s$ is
arbitrarily close to 0, then the sparsity assumption is useless. On
the other side of the spectrum, if $\mu_s \gg 0$, then the sparsity
assumption will turn out to be helpful.

Informally, we aim at replacing the dependency in the total number of
arms $d$ with the same dependency in the number $s$ of arms with
positive means. In other words, we wish to achieve an upper bound of
the following kind
\[
\mathbb{E}\left[  \operatorname{Reg}(T) \right] \lesssim \sum_{\substack{i\in [s]\\\Delta_i>0}} \frac{\log(T)}{\Delta_i}, 
\]
whenever this is possible. Notice that the above is precisely the
optimal regret bound \emph{if the agent knew in advance which are $s$ arms with positive means}. In the worst case, this  gives a distribution independent upper bound of the order of $\sqrt{sT}$  instead of the classical $\sqrt{dT}$ (up to logarithmic terms).
\bigskip

The Sparse Bandit problem is therefore a variation of the classical
stochastic multi-armed bandit problem (see \cite{bubeck2012regret} for
a survey) in which the agent knows \emph{the number of arms with positive means}.

There have been some works regarding sparsity assumptions in bandit problems. In the full information setting, some of them focus on \emph{sparse reward vectors}, i.e., at most $s$ components of $(X_1(t),\ldots,X_d(t))$ are positive (see for instance \cite{langford2009sparse,kwon2016gains}). Another considered problem is the one of \emph{sparse linear bandits} \cite{carpentier2012bandit,abbasi2012online,gerchinovitz2013sparsity,lattimore2015linear} in which the underlying unknown vector of parameter is assumed to be sparse -- that is with a constraint on its $L_1$ norm -- or even spiky in the \emph{crude and very specific case where $s=1$} \cite{bubeck2013bounded}.
However, none of the previously cited work tackles the following concrete problem. Assume that you are planning a marketing campaign for wich you have thousands of possible products to display. Most probably, many of them will be similar and have similar, very low expected returns but you do not have any possibility to know that in advance. In many common datasets such as Yandex's one \footnote{see https://www.kaggle.com/c/yandex-personalized-web-search-challenge}, depending on the query it is usual to have only 50 items out of 1500 that can be considered as relevant. Consequently, in order to avoid exploring thoses \emph{bad} items, you want to be able to set rules to eliminate them as quickly as possible and get a regret that scales in the number of \emph{good} arms.

\subsection{Contributions}
We introduce and investigate the sparse bandits problem by deriving
an asymptotic lower bound on the regret. We give an analogous result to the seminal bound of \cite{lai1985asymptotically}, and we construct an anytime algorithm \algosparse~that uses the optimistic principle of \cite{auer2002finite} together with the sparsity information available in order to reach optimal performance, up to constant terms.

Concretely, the lower bound that we prove distinguishes the possible
behavior of any \emph{uniformly efficient} algorithm according to the
value of the sparsity information available to the agent. To fix
ideas, assume that $\mu_1=1$ and for $2\leqslant i\leqslant s$, $\mu_i=\mu$, $\Delta_i = \Delta=1-\mu$. Then, we show that, if $\frac{d}{s}> \frac{\Delta}{\mu^2}+1$,  the sparsity of the problem is highly relevant so that regret is asymptotically lower-bounded as
$$
\liminf_{T\to +\infty}\frac{\Reg(T)}{\log(T)} \geq  \max\Big\{ \frac{s}{2\Delta},\  \frac{s\Delta}{2\mu^2}\Big\} =  \frac{s}{2\Delta}, \qquad \text{ if } \mu \geq \frac{1}{2}.
$$
The performance of the \algosparse~algorithm matches the lower bound
as it guarantees
\[
\Reg(T) \lesssim \max\Big\{ \frac{s\log(T)}{2\Delta},\  \frac{s\Delta\log(T)}{2\mu^2}\Big\} =  \frac{s\log (T)}{2\Delta}, \qquad \text{ if } \mu \geq \frac{1}{2}.
\]

%%% Local Variables:
%%% mode: latex
%%% TeX-master: "main_colt"
%%% End:

\section{The Stochastic Sparse Bandits Problem}
\label{sec:framework}
%!TEX root = main_colt.tex
%
We consider the classical \emph{stochastic} multi-armed bandit
problem, where a decision maker samples sequentially from $d \in \mathbb{N}$
i.i.d.\ processes $\left( (X_i(t))_{t\geqslant 1} \right)_{i\in
  [d]}$. We will keep denoting by  $\nu_i$ the probability distribution of $X_i(t)$
and $\mathbb{E}_{\nu_i}[X_i(t)] = \mu_i$ its mean. 
\medskip

The decision maker pulls at stage $t \geqslant 1$ an arm $I(t)\in
[d]$, and receives reward $X_i(t)$ which is his only observation
(specifically, he does not observe $X_i(t)$ for $i \neq I(t)$). The
(expected) cumulated reward of the decision
maker after $T\geqslant 1$ stages is then $\sum_{t=1}^T \mu_{I(t)}$
and his performance is evaluated through his \emph{regret}, defined as the difference between
the highest possible expected reward (had the means $\mu_1,\ldots,\mu_d$ been known in advance), and the actual reward. 
In other words:
\[\operatorname{Reg}(T) := T\mu_* - \sum_{t=1}^T X_{I(t)}(t), \quad \text{ where }  \mu_* = \max_{i \in [d]} \mu_i. \]
If we introduce the notations $\Delta_i = \mu_* - \mu_i$ and
$N_i(t):=\sum_{\tau=1}^{t-1}\mathds{1}\{I(\tau)=i\}$, the number of
times the decision maker pulled arm $i$ up to time $t-1$, then the
expected regret writes
\[ \mathbb{E}\left[  \operatorname{Reg}(T) \right]=\sum_{i=1}^d\Delta_i\mathbb{E}\left[ N_i(T+1) \right] \ . \]
This expression indicates that the (expected) regret should scale with $d$. And this is indeed the case without further assumption to leverage.

\begin{assumption} $s$ arms have positive means (i.e.,
  $\mu_i >0$), while the other $d-s$ arms have nonpositive means (i.e., $\mu_{i' }\leq 0$).
\end{assumption}
An arm with positive (resp.\ nonpositive) mean will also be refered to
as \emph{good} (resp.\ \emph{bad}).
In the remaining of the paper, we will assume, without loss of
generality and to simplify notation, that the means are re-ordered in
nonincreasing order:
$$\mu_{1}\geqslant  \mu_2 \geq \ldots \geq \mu_s >0 \geq \mu_{s+1} \geq \ldots \geq \mu_d\ .$$

We  will also denote  by $\overline{X}_i(n)$ the empirical mean of the
$n$ first
realizations of arm $i$ so that
\[ \overline{X}_i(N_i(t))=\frac{1}{N_i(t)}\sum_{\tau=1}^{t-1}\mathbbm{1}{\left\{ I(\tau) = i  \right\} }X_i(\tau), \]
with the convention that $\overline{X}_{i}(0)=0$. 
Besides,  we assume that the distributions $\nu_i$ are sub-Gaussians,
meaning that for all arm $i \in [d]$ and all $a>0$ and $t\geqslant 1$, we have
\[\mathbb{P}\left[ \left| X_i(t)-\mu_i  \right| >a \right] \leq 2e^{-a^2/2}. \]
For instance, this is the case if the $X_i(t)$ are assumed to be bounded, with support included in $[-1,1]$. Together with a Chernoff bound, one can
easily see that this implies for all arm $i \in [d]$ and all $a>0$ and $n\geqslant 1$,
\[ \mathbb{P}\left[ \overline{X}_i(n)-\mu_i\geqslant a \right]\leqslant e^{-na^2/2}.  \]

%%% Local Variables:
%%% mode: latex
%%% TeX-master: "main_colt"
%%% End:

% \section{Lower Bound}
% \label{sec:lowerbound}
% \input{LowerBound.tex}

\section{Lower Bound}
\label{sec:lowerbound}
%!TEX root = main_colt.tex
%
This section is devoted to proving a lower bound on the regret of any \emph{uniformly efficient} algorithm for the sparse bandit problem. To avoid too heavy expressions, the lower bound we establish holds for problems where the bad arms have a null expected reward, though handling general negative means does not require huge modifications from the given proof. Our goal is to provide a result that is easily generalizable to any stochastic bandit problem containing a sparsity information in the form of a threshold on the values of the expected return of the arms of interest. A generalization of the presented bound can be found in Appendix \ref{ap:LB2}.

\begin{definition}
	\label{def:uniformly_eff}
	An algorithm is \emph{uniformly efficient} if for any sparse
	bandit problem and all $\alpha\in(0,1]$, its expected regret satisfies
	$\mathbb{E}\left[   \operatorname{Reg}(T) \right] = o(T^\alpha)$.
\end{definition}

We state the bound for Gaussian bandit models with a 
fixed variance equal to $1/4$. In that case, a distribution is simply 
characterized by its mean $\mu$ and the Kullback-Leibler (KL) divergence between 
two models $\mu$ and $\mu'$ is equal to $2(\mu-\mu')^2$.
Consider
\[ \mathcal{S}(d,s)=\left\{ \mu=(\mu_1,\dots,\mu_d)\in
\mathbb{R}^d_+\,\middle|\, \text{$\mu$ has exactly $s$ positive components} \right\}.  \]

% \begin{theorem}
% 	\label{cor:LB}
% 	For a Gaussian sparse bandit problem $\nu =
% 	(\nu_1,\ldots,\nu_s,\nu_{s+1},\ldots,\nu_d) \in
% 	\mathcal{S}(d,s)\in \mathbb{R}^d$, an asymptotic lower bound on 
% 	the regret is given by the solution to the following linear optimization 
% 	problem:
% 	\begin{eqnarray}
% 	&  & f(\mu)\geq \quad \inf_{c\succeq 0} c_i \Delta_i\\
% 	s.t. & \quad\forall i\in \{2,...,s\},
% 	& 2 c_i \Delta_i^2\geq 1;\\
% 	& \forall i\in \{2,\ldots,s\},\, \forall j\in \{s+1,...d\},  & 2 c_j 
% 	\mu_1^2+ 2 c_i \mu_i^2 \geq 1\\
% 	& \forall i\in \{1,\ldots,d\}, & c_i \geq 0
% 	\end{eqnarray}
% 	whose solution can be computed explicitly and gives the following 
% 	problem-dependent lower bound:
% 	\begin{itemize}
% 		\item If $\frac{d-s}{\mu_1}- 
% 			\sum_{\substack{i\in [s]\\\Delta_i>0}}^{}\frac{\Delta_i}{\mu_i^2} > 0$,
% 		\begin{equation}
% 		\label{cor:LB0}
% 		\liminf_{T\to \infty}\frac{\operatorname{Reg}(T)}{\log(T)} \geq \sum_{\substack{i\in [s]\\\Delta_i>0}}^{}\max\Big\{ \frac{1}{2\Delta_i},\frac{\Delta_i}{2\mu_i^2}\Big\}
% 		\end{equation}
% 		\item otherwise, there exist $k\leq s$ such that $\frac{d-s}{\mu_1}- 
% 		\sum_{i=k}^{s}\frac{\Delta_i}{\mu_i^2} < 0$, and the lower bound is 
% 		\begin{equation}
% 		\label{cor:LB1}
% 		\liminf_{T\to \infty}\frac{\operatorname{Reg}(T)}{\log(T)} \geq \sum_{\substack{i\in [k]\\\Delta_i>0}}^{} 
% 		\frac{1}{2\Delta_i} + \sum_{i=k+1}^{s} 
% 		\frac{\mu_{k}^2}{\mu_i^2}\frac{\Delta_i}{2\Delta_{k}^2}
% 		+\frac{(d-s)}{2\mu_1}
% 		\left(1-\frac{\mu_{k}^2}{\Delta_{k}^2}\right).
% 		\end{equation}
% 	\end{itemize}

% \end{theorem}
\begin{theorem}
\label{cor:LB}
Let $\mu\in \mathcal{S}(d,s)$ such that its components are nonincreasing:
\[ \mu_1\geqslant \mu_2\geqslant \dots \geqslant \mu_s>\mu_{s+1}=\dots =\mu_d= 0, \]
and we denote $\Delta_i=\mu_1-\mu_i$ for all $i\in [d]$.
Then, for any uniformly efficient algorithm, played against arms whose
distributions are Gaussian with variance $1/4$ and with respective means $\mu_1,\dots,\mu_d$, one of the following asymptotic lower bounds hold.
\begin{itemize}
		\item If $\frac{d-s}{\mu_1}- 
			\sum_{\substack{i\in [s]\\\Delta_i>0}}^{}\frac{\Delta_i}{\mu_i^2} > 0$,
		\begin{equation}
		\label{cor:LB0}
		\liminf_{T\to \infty}\frac{\operatorname{Reg}(T)}{\log(T)} \geq \sum_{\substack{i\in [s]\\\Delta_i>0}}^{}\max\Big\{ \frac{1}{2\Delta_i},\frac{\Delta_i}{2\mu_i^2}\Big\}
		\end{equation}
		\item otherwise, there exist $k\leq s$ such that $\frac{d-s}{\mu_1}- 
		\sum_{i=k}^{s}\frac{\Delta_i}{\mu_i^2} < 0$, and the lower bound is 
		\begin{equation}
		\label{cor:LB1}
		\liminf_{T\to \infty}\frac{\operatorname{Reg}(T)}{\log(T)} \geq \sum_{\substack{i\in [k]\\\Delta_i>0}}^{} 
		\frac{1}{2\Delta_i} + \sum_{i=k+1}^{s} 
		\frac{\mu_{k}^2}{\mu_i^2}\frac{\Delta_i}{2\Delta_{k}^2}
		+\frac{(d-s)}{2\mu_1}
		\left(1-\frac{\mu_{k}^2}{\Delta_{k}^2}\right).
		\end{equation}
	\end{itemize}

\end{theorem}

\textbf{Remarks.} Since the decision maker has more knowledge on the parameters of the problem than in the classical multi-armed bandit problem, we expect the lower bound to be less than the traditional one (without the sparsity assumption), which is 
$$
\liminf_{T\to \infty}\frac{\operatorname{Reg}(T)}{\log(T)} \geq \sum_{\substack{i\in [d]\\\Delta_i>0}}^{} \frac{1}{2\Delta_i}\ .
$$
Indeed, since the $\max$ is smaller than the sum, in the first case, we have
$$ \sum_{\substack{i\in [s]\\\Delta_i>0}}^{}\max\Big\{ \frac{1}{2\Delta_i},\frac{\Delta_i}{2\mu_i^2}\Big\} \leq  \sum_{\substack{i\in [s]\\\Delta_i>0}}^{} \frac{1}{2\Delta_i} + \sum_{\substack{i\in [s]\\\Delta_i>0}}^{} \frac{\Delta_i}{2\mu_i^2}\leq   \sum_{\substack{i\in [s]\\\Delta_i>0}}^{} \frac{1}{2\Delta_i} +\frac{d-s}{\mu_1}= \sum_{\substack{i\in [d]\\\Delta_i>0}}^{} \frac{1}{2\Delta_i}\ .
$$
Similarly, in the second case, we get that
\begin{align*}
 \sum_{i=1}^{k} 
\frac{1}{2\Delta_i} + \sum_{i=k+1}^{s} 
\frac{\mu_{k}^2}{\mu_i^2}\frac{\Delta_i}{2\Delta_{k}^2}  +\frac{(d-s)}{2\mu_1}\left(1-\frac{\mu_{k}^2}{\Delta_{k}^2}\right) & =
 \sum_{\substack{i\in [d]\\\Delta_i>0}}^{} \frac{1}{2\Delta_i} - \frac{\mu_k^2}{\Delta_k^2}\Big(\underbrace{\frac{d-s}{2\mu_1}-\sum_{i=k+1}^s \frac{\Delta_i}{2\mu_i^2}}_{> 0}\Big)-\sum_{i=k+1}^s \frac{1}{2\Delta_i} .
\end{align*}
Moreover, if $\frac{\Delta_s}{\mu_s^2} > \frac{d-s}{\mu_1}$, then both lower bounds match. Stated otherwise, the sparsity assumption is irrelevant as soon as
$$
\mu_s \leq \mu_1 \frac{-1 + \sqrt{1+4(d-s)}}{2(d-s)} \simeq \frac{\mu_1}{\sqrt{d-s}}\ .
$$
 
\begin{proof}
The proof relies on changes of measure arguments originating from \cite{graves1997asymptotically}.
% Let $\mu\in \mathcal{S}(d,s)$ such that its components are nonincreasing
% \[ \mu_1\geqslant \mu_2\geqslant \dots \geqslant \mu_s>\mu_{s+1}=\dots =\mu_d= 0. \]
% Denote also $\Delta_i=\mu_1-\mu_i$ for all $i\in [d]$. For $i\in \left\{ s+1,\dots,d \right\}$ we have $\Delta_i=\mu_1$.	
First, consider the set of changes of distributions that modify the
best arm without changing the marginal of the best arm in the original 
sparse bandit problem:
% Namely, for a given $\mu\in \mathcal{S}(d,s)$, we consider the
% following set of acceptable changes of measure
\[ \mathcal{B}(\mu)=\left\{ \mu'\in \mathcal{S}(d,s)\,\middle|\,\mu_1'=\mu_1
\text{ and }\max_{i\in [d]}\mu_i<\max_{i\in [d]}\mu'_i  \right\}.  
\]
Concretely, if one considers an alternative sparse bandit model $\mu'$ such 
that  one
of the originally null arms becomes the new best arm, then one of the
originally non-null arms in $\mu$ must be taken to zero in $\mu'$ in order 
to keep the
sparse structure of the problem.

	In general, the equivalent of Th.~17 in \citep{kaufmann2015complexity}
	or Proposition~3 in \citep{lagree2016multiple} can be stated in our case
	as follows: For all changes of measure $\mu' \in \mathcal{B}(\mu)$,

\begin{equation}
\label{eq:constraint}
\liminf_{T\to \infty} \dfrac{\sum_{i=1}^{d} 2 \mathbb{E}
	[N_i(T+1)](\mu_i-\mu'_i)^2}{\log(T)} \geq 1.  	
\end{equation}
Details on this type of informational lower bounds can be found in \cite{garivier2016explore} and references therein.
Now, following general ideas from \cite{graves1997asymptotically} and
lower bound techniques from
\cite{lagree2016multiple} and \cite{combes2015combinatorial}, we may give a
variational form of the lower bound on the regret satisfying the above 
constraint.

\begin{equation}
\label{eq:opt_problem}
\liminf_{T\to \infty} \frac{\operatorname{Reg}(T)}{\log (T)} \geq \inf_{c\in \mathcal{C}} \sum_{i\in [d]}^{}c_i\Delta_i ,
\end{equation}
where the set $\mathcal{C}$ corresponds to constraints that are directly implied by Eq.\eqref{eq:constraint} above:

\[ \mathcal{C}=\left\{ (c_i)_{i\in [d]}\in \mathbb{R}_+^d\,\middle|\, 
\forall \mu'\in \mathcal{B}(\mu),\ 2\sum_{i\in [d]}^{}c_i(\mu_i-\mu_i')^2\geqslant 1 \right\}.  \]

We aim at obtaining a lower bound of the infimum from Eq.\eqref{eq:opt_problem}.
In this constrained optimization problem, the constraints set is
huge because every change of measure $\mu'\in\mathcal{B}(\mu)$ must satisfy 
\eqref{eq:constraint}. On
the other side, relaxing some constraints -- or considering only a
subset of $\mathcal{B}(\mu)$-- simply allows to reach even lower 
values 
\footnote{At that point, we may lose the optimality of the finally 
obtained lower bound but this is a
price we accept to pay in order to obtain a computable solution. }.

\begin{figure}
\centering
\includegraphics[width=0.7\linewidth]{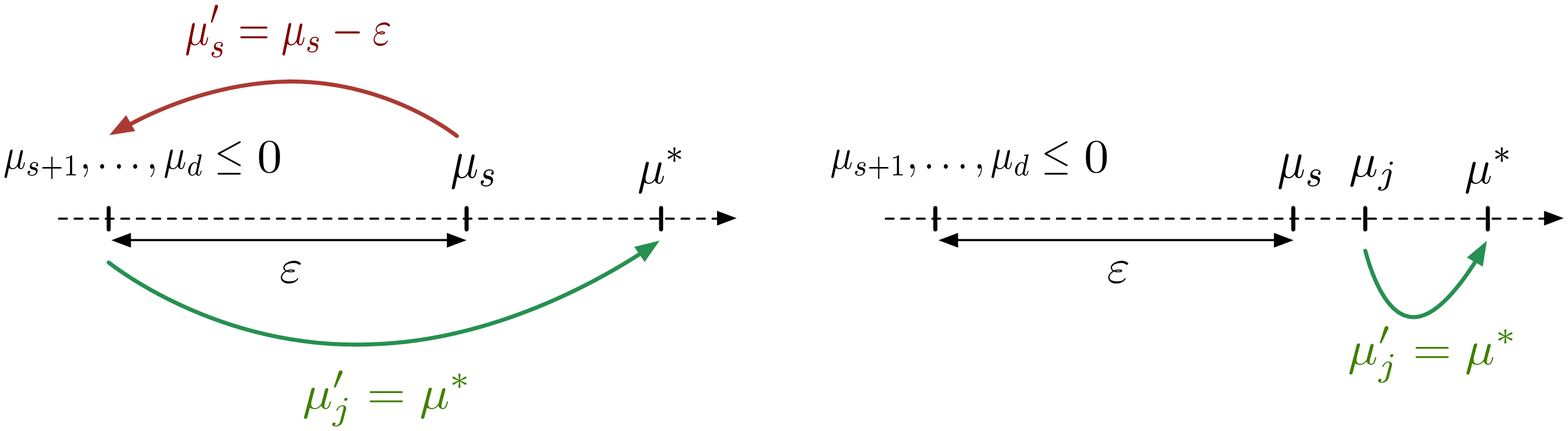}
\caption{Illustration of the various changes of distribution considered in $\mathcal{B}(\mu)$}
\label{fig:changes}
\end{figure}

We  consider $\tilde{\mathcal{C}}$ defined as
\[ \tilde{\mathcal{C}}=\left\{ (c_i)_{i\in [d]}\in \mathbb{R}_+^d\,\middle|\,
\text{for all $i \in [s]\setminus \left\{ 1 \right\}$ and $j\in [d]\setminus [s]$},\quad 
\begin{aligned}
&c_i\Delta_i^2\geqslant 1/2\\
&c_j\mu_1^2+c_i\mu_i^2\geqslant 1/2
\end{aligned}
   \right\}  \]
and we prove that $\mathcal{C}$ is a subset of $\tilde{\mathcal{C}}$, 
namely that there are more acceptable vectors of coefficients in $\mathcal{C}$ allowing us to reach lower values of the argument.

Let $(c_i)_{i\in [d]}\in \mathcal{C}$ and let us prove that it belongs to $\tilde{\mathcal{C}}$. 
Let $i\in [s]\setminus \left\{ 1 \right\}$,  $\gamma>0$ and consider $\mu^{(i,\gamma)}\in \mathbb{R}^d$ defined as the following modification of $\mu$:
\[ \mu_k^{(i,\gamma)}=
\begin{cases}
\mu_1+\gamma &\text{if $k=i$}\\
\mu_k&\text{otherwise},
\end{cases}\quad k\in [d]. \]
We easily see that $\mu^{(i,\gamma)}$ belongs to $\mathcal{B}(\mu)$. Therefore, by definition of $\mathcal{C}$, $(c_i)_{i\in [d]}$ satisfies:
\[ \sum_{k\in [d]}^{} c_k(\mu_k-\mu_k^{(i,\gamma)})^2\geqslant \frac{1}{2}, \]
which, by definition of $\mu^{(i,\gamma)}$ boils down to $c_i(\Delta_i+\gamma)^2\geqslant 1/2$. 
This being true for all $\gamma>0$, we have $c_i\Delta_i^2\geqslant 1/2$. The first condition in the defintion of $\tilde{\mathcal{C}}$ is then satisfied.

Similary, for $i\in [s]\setminus \left\{ 1 \right\}$, $j\in [d]\setminus [s]$ and $\gamma>0$ we consider $\mu^{(i,j,\gamma)}\in \mathbb{R}^d$ defined by:
\[ \mu_k^{(i,j,\gamma)}=
\begin{cases}
0 &\text{if $k=i$}\\
\mu_1+\gamma &\text{if $k=j$}\\
\mu_k&\text{otherwise},
\end{cases}\quad k\in [d]. \]
$\mu^{(i,j,\gamma)}$ also belongs to $\mathcal{B}(\mu)$. By definition of $\mathcal{C}$, $(c_i)_{i\in [d]}$ satisfies:
\[ \sum_{k\in [d]}^{} c_k(\mu_k-\mu_k^{(i,j,\gamma)})^2\geqslant \frac{1}{2}, \]
which boils down to $c_i\mu_i^2+c_j\Delta_j^2\geqslant 1/2$ (after taking the infinimum over $\gamma>0$). Since $\Delta_j=\mu_1$, the second condition in the definition of $\tilde{\mathcal{C}}$ is satisfied. We have proved that $(c_i)_{i\in [d]}$ belongs to $\tilde{\mathcal{C}}$ and consequently that
 $\mathcal{C}$ is a subset of $\tilde{\mathcal{C}}$. Therefore,
\[ \liminf_{T\to +\infty}\frac{\operatorname{Reg}(T)}{\log (T)}\geqslant \inf_{c\in \mathcal{C}}\sum_{i\in [d]}^{}c_i\Delta_i\geqslant \inf_{c\in \tilde{\mathcal{C}}}\sum_{i\in [d]}^{}c_i\Delta_i. \]
The computation of the optimization problem over $\tilde{\mathcal{C}}$ is deferred to Appendix~\ref{ap:LBproof}.
\end{proof}

\begin{remark}
		This is a linear optimization problem under inequality constraints so 
there exist algorithmic methods such as the celebrated Simplex 
algorithm \cite{dantzig2016linear} to compute a numeric solution of it. 
Nonetheless, in our case, it is 
possible to give an explicit solution.
\end{remark}

%%% Local Variables:
%%% mode: latex
%%% TeX-master: "main_colt"
%%% End:

\section{Sparse UCB}
\label{sec:algorithm}
% \input{Algorithm.tex}
%!TEX root = main_colt.tex
%

\subsection{The \algosparse~ algorithm}

The \algosparse~algorithm is formally defined in Algorithm~\ref{algo}
but let us first provide   an
informal description. The algorithm can be in different \emph{phases}
(denoted $\mathfrak{r},\mathfrak{f}$ and $\mathfrak{u}$), depending on
past observations, and its behavior radically changes from one phase to another. At each time $t\geqslant 1$, the variable
$\omega(t)\in \left\{ \mathfrak{r},\mathfrak{f},\mathfrak{u} \right\}$
will specify the phase the algorithm is in. The variable is not useful for the algorithm itself, but will be handy for reference in the analysis. We now describe the different phases.

% \todov{This is not how the algorithm works... Whenever the set $\mathcal{K}(t)$ does not contain $s$ arms, then we sample from $\mathcal{J}(t) \backslash \mathcal{K}(t)$.

% I would rather define $\mathcal{K}(t)$ as the set of arms in $\mathcal{J}(t)$ whose number of samples  is not big enough so that their means "are not statistically  significantly positive".}
% \todov{You have to give names here, explicitly, to the different "epochs" - initialization, ucb, force-log, etc. - and provide the notations we will use later on}
%The algorithm takes as inputs the number of arms $d$ and the number $s$ of arms with positive means. 

\begin{description}
\item[Round-robin]{The algorithm starts with a \emph{round-robin} phase, which corresponds to $\omega(t)=\mathfrak{r}$. Each of the $d$ arms is pulled once successively.

Then, for each time $t\geqslant d+1$, the following sets are defined:
\begin{align*}
\mathcal{J}(t)&:=\left\{ i\in [d]\,\middle|\,\overline{X}_i(N_i(t))\geqslant 2\sqrt{\frac{\log (N_i(t))}{N_i(t)}} \right\},\\ 
\mathcal{K}(t)&:=\left\{ i\in [d]\,\middle|\, \overline{X}_i(N_i(t))\geqslant 2\sqrt{\frac{\log (t)}{N_i(t)}} \right\}.
\end{align*}
We will refer to the arms in $\mathcal{J}(t)$ as the \emph{active arms} and those in $\mathcal{K}(t)$ as \emph{active and sufficiently sampled}.
  
If there are less than $s$  active arms, i.e., $|\mathcal{J}(t)|< s$,  the algorithm
enters a \emph{round-robin} phase, and pulls each arm successively. This implies that $\omega(t)=\mathfrak{r}$ for the
next $d$ stages.}
\item[Force-log]{If there are at least $s$ active arms, but less than $s$ sufficiently sampled arms ($|\mathcal{K}(t)| < s $), the algorithm enters a \emph{force-log} phase ($\omega(t)=\mathfrak{f}$). In this phase, the algorithm pulls any arm in the set $\mathcal{J}(t)\setminus \mathcal{K}(t)$.}
\item[UCB]{If the set $\mathcal{K}(t)$ contains at least $s$ arms, the
    algorithm enters a \emph{UCB} phase
    ($\omega(t)=\mathfrak{u}$). The algorithm selects an arm in
    $\mathcal{K}(t)$ according to the UCB rule, i.e.\ it chooses the
    arm $i\in \mathcal{K}(t)$ which maximizes the quantity:
\[ \overline{X}_i(N_i(t))+2\sqrt{\frac{\log (t)}{N_i(t)}}. \]}
\end{description}
The pseudo-code of the whole procedure is given in Algorithm~\ref{algo} and
the skeleton in Figure~\ref{fig:skeleton}.

\begin{algorithm}[H]
  \caption{{\algosparse}}
  \label{algo}
  \SetKw{Init}{Initialization:}
  
  \DontPrintSemicolon % Some LaTeX compilers require you to use
  % \dontprintsemicolon instead
  \KwIn{the total number of arms $d$ and the number of arms with
          positive means $s$}
  \Init{$t\gets 1$}\;~\\
        \For(\tcc*[f]{round-robin}){$k=1\dots d$}{
          $I(t)\gets k$\; ;$\quad$
          $\omega(t)\gets \mathfrak{r}$\; ;$\quad$
          $t\gets t+1$\; ;
        }
  \While{$t \leq T$} {
          Compute $\mathcal{J}(t) \gets \left\{ i\in [d]\,\middle|\, \overline{X}_i(N_i(t))\geq 2\sqrt{\frac{\log (N_i(t))}{N_i(t)}} \right\}$ \; ~ \\Compute $\mathcal{K}(t) \gets \left\{ i\in [d]\,\middle|\, \overline{X}_i(N_i(t))\geq 2\sqrt{\frac{\log (t)}{N_i(t)}} \right\}$ \; ~ \\
          \uIf{$\vert \mathcal{J}(t) \vert < s$}{
            \For(\tcc*[f]{round-robin}){$k=1\ldots d$}{
              $I(t) \gets k$\; ;$\quad$
              $\omega(t)\gets \mathfrak{r}$\; ;$\quad$
              $t \gets t+1$\; ;
            }
          }
          \uElseIf(\tcc*[f]{force-log}){$\left| \mathcal{K}(t) \right| <s$}{
                $I(t)\in  \mathcal{J}(t)\setminus \mathcal{K}(t)$\; ;$\quad$
                $\omega(t)\gets \mathfrak{f}$\; ;$\quad$
                $t\gets t+1$\; ;
            % \For{$k\in \mathcal{J}(t)$}{
            %   \While{$\overline{X}_k(N_k(t))<2\sqrt{(\log(t))/N_k(t)}$}{
            %     $\omega(t)\gets \mathfrak{f}$\;\\
            %     $I(t)\gets k$\;\\
            %     $t\gets t+1$\;
            %   }
            % }
          }
          \Else(\tcc*[f]{UCB}){
            % {\sc Ucb} on $\mathcal{J}(t)$:\; \\
            $I(t) \in \arg\max_{i\in \mathcal{K}(t)}\left\{   \overline{X}_i(N_i(t)) + 2\sqrt{\frac{\log(t)}{N_i(t)}} \right\}$\; ;$\quad$
$\omega(t)\gets \mathfrak{u}$\; ;$\quad$
            $t\gets t+1$\; ; }}
\end{algorithm}

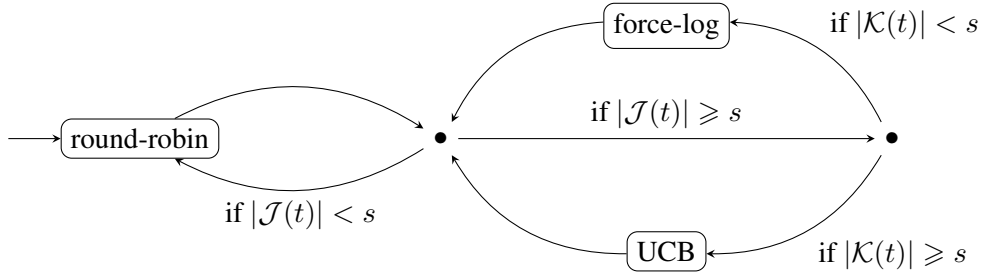
\begin{figure}
  \centering
\begin{tikzpicture}[>=stealth,box/.style={draw,rounded corners}]

 % \node (start) at (-2,0) {start}; 
 \node[box] (round-robin) at (0,0) {round-robin};
 \node (n1) at (4,0) {$\bullet$};
 \node (n2) at (10,0) {$\bullet$};
 \node[box] (force-log) at (7,1.5) {force-log};
 \node[box] (UCB) at (7,-1.5) {UCB};

 \draw [->] (round-robin)[right] to[bend left]  (n1);
 \draw [->] (n1)[right] to[bend left] node[below]{if $\left|
     \mathcal{J}(t) \right| <s$} (round-robin);
 \draw[->] (n1) to node[above]{if $\left| \mathcal{J}(t) \right| \geqslant s$} (n2) ;

 \draw[->] (n2) to[bend right] node[above right]{if $\left| \mathcal{K}(t) \right| <s$} (force-log);
 \draw[->] (n2) to[bend left] node[below right]{if $\left| \mathcal{K}(t) \right| \geqslant s$} (UCB);
 \draw[->] (force-log) to [bend right] (n1);
 \draw[->] (UCB) to [bend left] (n1);
 \draw[->] (-1.75,0) to (round-robin);
\end{tikzpicture}
\caption{Skeleton of the \algosparse~algorithm. Each $\bullet$ corresponds to a conditional statement, where each departing arrow corresponds to a condition (which is written midway).}
\label{fig:skeleton}
\end{figure}

The broad idea is that the algorithm should quickly identify the $s$
\emph{good} arms, and then pull those arms according to an
\emph{UCB} rule (or, alternatively, any other policy). At the end, only those $s$ \emph{good} arms would be pulled an infinite number of times.

The set $\mathcal{J}(t)$ of active arms is defined in such a way that
the expected number of pulls needed for a good arm to become active
is finite. Therefore, only a finite number (in expectation) of
\emph{round-robin} phases is needed for all $s$ good arms to become
active (see Lemma~\ref{lm:round-robin}). 

Reciprocally, a bad arm (with non-positive mean) is only pulled while active, that is a finite number of times in expectation. The main issue occurs when a bad arm happens to be active. In that case, the delay between two successive pulls of an active null arm typically increases exponentially fast because the regret scales with $\log(t)$. Consequently, it would take an exponential number of stages for this arm to become inactive again. And this could be dramatic for the regret if the best arm was, at the same time, inactive, as the regret would increase by a fixed constant of at least $\Delta_2$ on all those stages.
The purpose of the \emph{force-log} phases is to make sure that each active arm gets pulled
sufficiently often so that the expected number of steps a bad arm
remains active is finite. If the best arm happened to be inactive, then the number of active arms would drop
below $s$, and performing a \emph{round-robin} phase would allow it to quickly become active again.

% \bigskip

% The following statement gives an upper bound on the regret guaranteed
% by the algorithm.
\begin{theorem}
\label{thm:regret-bound}
The \algosparse~algorithm guarantees
\[\mathbb{E}\left[   \operatorname{Reg}(T) \right] \lesssim \log(T) \sum_{\substack{i\in [s]\\\Delta_i>0}} \left( \frac{1}{\Delta_i}+\frac{\Delta_i}{\mu_i^2} \right),  \]
where notation $\lesssim$ removes universal multiplicative constants
and additive data-dependant constants. The detailed statement can be
found in Appendix~\ref{sec:analys-algosp-algor}.
\end{theorem}

\subsection{Sketch of the proof}
We decompose the event $\left\{ I(t)=i \right\}$ of a good arm $i\in
[s]$ being pulled at time $t\geqslant 1$ with respect to the different phases of \algosparse:
\begin{equation}
  \label{eq:decomposition}
  \left\{ I(t)=i \right\}=R_i(t)\sqcup F_i(t) \sqcup U_i(t) \sqcup V_i(t), \quad t\geqslant 1,
\end{equation}
where the different events are defined as follows:
\begin{itemize}
\item $R_i(t):=\left\{ I(t)=i,\ \omega(t)=\mathfrak{r} \right\}$ is
  event of arm $i$ being pulled at time $t$ during a \emph{round-robin} phase; 
\item  $F_i(t):=\left\{ I(t)=i,\ \omega(t)=\mathfrak{f} \right\}$ is
  the event of arm $i$ being pulled at time $t$ during a \emph{force-log} phase;
\item $U_i(t):=\left\{ I(t)=i,\ \omega(t)=\mathfrak{u},\ 1\in \mathcal{K}(t)
\right\}$ is the event of arm $i$ being pulled at time $t$ during a \emph{UCB}
phase while the optimal arm is active and sufficiently sampled; 
\item 
$V_i(t):=\left\{ I(t)=i,\ \omega(t)=\mathfrak{u},\ 1\not \in
  \mathcal{K}(t) \right\}$ is the event of arm $i$ being pulled at
time $t$ during a \emph{UCB} phase while the optimal arm is not active or not sufficiently sampled.
\end{itemize}

For the bad arms $i \in \{s+1,\ldots,d\}$, we consider a simpler decomposition. For $t\geqslant d+1$, we introduce $A_i(t):=\left\{ I(t)=i,\  i\in
  \mathcal{J}(t) \right\}$, which is the event of arm $i$ being pulled
at time $t$ while active, so that
\[ \left\{ I(t)=i \right\}=R_i(t)\sqcup A_i(t),\quad t\geqslant 1,    \]
see, e.g. Property (\ref{item:subset}) from Lemma~\ref{lm:by-definition}.
\bigskip

%We aim at establishing an upper bound on the expected regret up to
%time $T$. We decompose the latter quantity as follows.
%\[\mathbb{E}\left[ \operatorname{Reg}(T) \right]=\sum_{\substack{i\in [d]\\\Delta_i>0}}^{}\Delta_i\mathbb{E}\left[ N_i(n+1) \right]=\sum_{\substack{i\in [d]\\\Delta_i>0}}^{}\Delta_i\mathbb{E}\left[ \sum_{t=1}^T\mathbbm{1}\left\{ I(t)=i \right\}  \right].     \]
%For $i\in \left\{ 2,\dots,s \right\}$, we use the decomposition of the
%event $\left\{ I(t)=i \right\}$ given in
%Equation~(\ref{eq:decomposition}); for $i\in \left\{ s+1,\dots,d
%\right\}$, We then get
Using the above decompositions, we can write the regret as:
\begin{align*}
\mathbb{E}\left[ \operatorname{Reg}(T) \right]&=\sum_{\substack{i\in [d]\\\Delta_i>0}}\Delta_i\mathbb{E}\left[
  \sum_{t=1}^T\mathbbm{1}\left\{ R_i(t) \right\}
  \right]+\sum_{\substack{i\in [s]\\\Delta_i>0}}\Delta_i \mathbb{E}\left[
  \sum_{t=1}^T\mathbbm{1}\left\{ F_i(t) \right\}
                              \right]+\sum_{\substack{i\in [s]\\\Delta_i>0}}^{}\Delta_i\mathbb{E}\left[
                              \sum_{t=1}^T\mathbbm{1}\left\{ U_i(t)
                              \right\}  \right]    \\  
&\qquad\qquad  +\sum_{\substack{i\in [s]\\\Delta_i>0}}\Delta_i\mathbb{E}\left[ \sum_{t=1}^T\mathbbm{1}\left\{ V_i(t) \right\}  \right]+\sum_{i=s+1}^d\Delta_i\mathbb{E}\left[ \sum_{t=1}^T\mathbbm{1}\left\{ A_i(t) \right\}  \right].  
\end{align*}

We upper-bound independently the five above quantities. For the sake
of clarity, we only provide here the main ideas of proof. A detailed
analysis can be found in Appendix~\ref{sec:analys-algosp-algor}.

\setcounter{theorem}{5}

%\begin{lemma}
%  \label{lm:round-robin}
%  For $i\in [d]$, the number of times arm $i$ is pulled, while the
%  algorithm is performing a \emph{round-robin} process, is bounded in expectation as:
%  \[ \mathbb{E}\left[ \sum_{t=1}^{+\infty}\mathbbm{1}\left\{ R_i(t)
%      \right\}\right]\leqslant 1+3s+ \sum_{j=1}^s\frac{1}{\mu_j^2}\left( 8+32\log \left(\frac{16}{\mu_j^2}\right) \right).    \]
%\end{lemma}
\begin{lemma}
  \label{lm:round-robin}
  The regret induced the \emph{round-robin} phases is controlled by:
  \[ \mathbb{E}\left[ \sum_{t=1}^{+\infty}\mathbbm{1}\left\{ R_i(t)
      \right\}\right]\leqslant 1+3s+ 8\sum_{j=1}^s\frac{1}{\mu_j^2}\left( 1+4\log \left(\frac{16}{\mu_j^2}\right) \right),\quad i\in [d].    \]
\end{lemma}
\textbf{Main argument of proof.}
% Necessarily, the algorithm is in the \emph{round-robin} phase if, at least, one of the arm $i \in [s]$ is below its threshold 
% $2\sqrt{\frac{\log(t)_i(t)}{N_i(t)}}$, which happens with a
% probability exponentially small if the latter is smaller than
% $\mu_i/2$, i.e., if $N_i(t)$ is at least  of the order of
% $\frac{\log(1/\mu_i^2)}{\mu_i^2}$.
% The quantity is the expected number of times a given arm
% $i\in [d]$ is pulled during a \emph{round-robin}
% process. Each arm being pulled exactely once during a \emph{round-robin}
% process, this correponds to the expected number of \emph{round-robin}
% processes. We show that this quantity is finite (i.e.\ independent of
% $T$).
The algorithm performs a \emph{round-robin} phase only if less than
$s$ arms are active, thus necessarily when one of the good arms $j\in [s]$ is not active. This implies that 
$\overline{X}_j(N_j(t))<2\sqrt{\frac{\log (N_j(t))}{N_j(t)}}$. The probability of
this happening decreases exponentially fast and as a consequence, the expected number of \emph{round-robin} phases is bounded.
\hfill $\Box$
%\begin{lemma}
%  \label{lm:force-log}
%  For $i\in [s]$ and $n\geqslant d+1$, the number of times arm $i$
%  is pulled up to time $T$ during 
%  \emph{force-log} processes, is bounded in expectation as:
%  \[ \mathbb{E}\left[ \sum_{t=1}^{T} \mathbbm{1}\left\{ F_i(t)
%      \right\}  \right]\leqslant \frac{16\log (T)+8}{\mu_i^2}.  \]
%\end{lemma}
\begin{lemma}
  \label{lm:force-log}
  The regret induced by a good arm $i\in [s]$ during \emph{force-log} phases is controlled by:
  \[ \mathbb{E}\left[ \sum_{t=1}^{T} \mathbbm{1}\left\{ F_i(t)
      \right\}  \right]\leqslant \frac{16\log (T)+8}{\mu_i^2}.  \]
\end{lemma}
\textbf{Main argument of proof.}
% Arm $i \in [s]$  is pulled in force
% log if  $\overline{X}_i(N_i(t))<2\sqrt{\frac{\log(t)}{N_i(t)}}$, hence
% more or less if  $\mu_i<2\sqrt{\frac{\log(t)}{N_i(t)}}$ which
% indicates that $N_i(t)$ increases as $\frac{\log(t)}{\mu_i^2}$.
% The quantity in the left-hand side is the expected number of times a given arm $i\in
% \left\{ 2,\dots,s \right\}$ (therefore with positive mean) is pulled
% during a \emph{force-log} process. We establish a bound on this
% quantity of order $\log (T)/\mu_i^2$.
Arm $i \in [s]$ is pulled during a \emph{force-log} phase if its empirical mean  is below $2\sqrt{\frac{\log (t)}{N_i(t)}}$.
Because arm $i$ has a positive mean, the probability of this happening
turns out to decrease exponentially, as soon as $N_i(t)\geqslant \frac{16\log (T)}{\mu_i^2}$. 
Therefore, the expected number of times arm $i$ is pulled during a
\emph{force-log} phase is bounded by $16\frac{\log (T)}{\mu_i^2}$ plus a
constant term.
\hfill $\Box$

\begin{lemma}
  \label{lm:ucb}
The regret induced by a good arm $i\in [s]$ during \emph{UCB} phases, while $1 \in \mathcal{K}(t)$, is controlled by
  \[ \mathbb{E}\left[  \sum_{t=1}^T\mathbbm{1}\left\{ U_i(t) \right\}  \right]\leqslant \frac{16\log (T)+8}{\Delta_i^2}+3.  \]
\end{lemma}
\textbf{Main argument of proof.} The proof basically follows the steps of the classic UCB analysis by \cite{auer2002finite}. \hfill $\Box$
% The third expectation is the expected number of times a given arm
% $i\in \left\{ 2,\dots,s \right\}$ is pulled during an \emph{UCB}
% process. The classical argument gives an upper bound on this quantity
% of order $\log (T)/\Delta_i^2$ (see Lemma~\ref{lm:ucb} for details).

\begin{lemma}
  \label{lm:wrong-ucb}
 The regret induced by good arms during \emph{UCB} phases, while $1 \not\in \mathcal{K}(t)$, is controlled by:
  \[ \sum_{i\in [s]}^{}\Delta_i\mathbb{E}\left[ \sum_{t=1}^{T}\mathbbm{1}\left\{ V_i(t) \right\}  \right]\leqslant \frac{d\Delta_s\pi^2}{6}.  \]
\end{lemma}
\textbf{Main argument of proof.}
% If the algorithm is in the UCB phase while arm $1$ is not active, then an arm $i \in \{s+1,\ldots,d\}$ must be active and its empirical average must be above $2\sqrt{\frac{\log(t)}{N_i(t)}}$, which happens with a probability of the order of $t^{-2}$.
% The third term of the above sum deals with the expected number of
% times a given good arm $i\in \left\{ 2,\dots,s \right\}$ is pulled during a \emph{UCB}
% process and while the best arm does not belong to $\mathcal{K}(t)$. We show
% that this quantity is finite using the following argument (see
% Lemma~\ref{lm:wrong-ucb} for details).
The algorithm performs a \emph{UCB} phase if the set $\mathcal{K}(t)$ has at least $s$ arms. If it does not contain the best arm $1$,  it must necessarily contain an arm $j\in \left\{ s+1,\dots,d \right\}$, i.e.\ with
nonpositive mean. As a consequence, the empirical mean of arm $j$ is
above $2\sqrt{\frac{\log (t)}{N_j(t)}}$. Because arm $j$ has a nonpositive
mean, the probability of this happening turns out to decrease as $t^{-1/2}$. Consequently, the expected number of times a good arm is pulled
during a \emph{UCB} phase while the best arm does not belong to
$\mathcal{K}(t)$ is finite. \hfill $\Box$

\begin{lemma}
  \label{lm:bad-arm-above}
  The regret induced by a bad arm $i\in \left\{ s+1,\dots,d
  \right\}$ while active is controlled by:
  \[ \mathbb{E}\left[ \sum_{t=1}^{+\infty}\mathbbm{1}\left\{ A_i(t) \right\}  \right]\leqslant \frac{\pi^2}{6}.  \]
\end{lemma}
\textbf{Main argument of proof.}
% At each stage where arm $1$ is not active, and arm $i \in \{s+1,\ldots,d\}$ is active and pulled, the empirical mean of the latter must be bigger than the threshold $2\sqrt{\frac{\log(N_i(t))}{N_i(t)}}$. Hoeffding inequality indicates that  the total expected number of pulls in this case is upper-bounded by $\sum_{t \geq 1} 1/t^2 \leq \pi^2/6$.
% The last term deals with the expected number of times a given bad arm
% $i\in \left\{ s+1,\dots,d \right\}$ is pulled during a
% \emph{force-log} or a \emph{UCB} process. We show that this quantity
% is finite using the following argument (see
% Lemma~\ref{lm:bad-arm-above} for details).
Arm $i$ is active if its empirical mean is above $2\sqrt{\frac{\log (N_i(t))}{N_i(t)}}$. Because its mean is  nonpositive, this happens with a total  probability of the order of $\sum t^{-2}$. Therefore, the regret incurred when $i$ is active is bounded. \hfill $\Box$

\bigskip

It only remains to combine the above results, to  upper bound  the expected regret  as
\[  \mathbb{E}\left[  \operatorname{Reg}(T) \right] \lesssim \log (T)\sum_{\substack{i\in [s]\\\Delta_i>0}}\left(  \frac{1}{\Delta_i}+\frac{\Delta_i}{\mu_i^2} \right) +d\sum_{j=1}^s \frac{\mu_1\log(1/\mu_j^2)}{\mu^2_j},  \]
where we omitted multiplicative universal constants. We emphasize the fact that the last term is independent of $T$, hence the dominating term is 
\[
\mathbb{E}\left[  \Reg(T) \right] \lesssim  \log (T)\sum_{\substack{i\in [s]\\\Delta_i>0}} \max \left\{ \frac{1}{\Delta_i},\frac{\Delta_i}{\mu_i^2} \right\}.
\]

%%% Local Variables:
%%% mode: latex
%%% TeX-master: "main_colt"
%%% End:

\section{Optimality, ranges of sparsity and  constants optimization}
\label{sec:optimality}
%!TEX root = main_colt.tex
%

We prove in this section that the algorithm \algosparse~ is optimal, up to multiplicative factor, for a wide range of parameters. For this purpose, we recall the different bounds we obtained, up to multiplicative universal constants and additive data-dependent constants.

\subsection{Strong sparsity}
The regime of strong sparsity is attained when $\frac{d-s}{\mu_1}- \sum_{\substack{i\in [s]\\\Delta_i>0}}^{}\frac{\Delta_i}{\mu_i^2} > 0$. In that case, the lower bound of  Theorem~\ref{cor:LB} rewrites as 
$$
\liminf_{T\to \infty}\frac{\Reg(T)}{\log(T)} \gtrsim \sum_{\substack{i\in [s]\\\Delta_i>0}}^{}\max\Big\{ \frac{1}{\Delta_i},\frac{\Delta_i}{\mu_i^2}\Big\}
$$
while \algosparse~ suffers an expected regret bounded as
$$
\frac{\Reg(T)}{\log(T)} \lesssim  \sum_{\substack{i\in [s]\\\Delta_i>0}}^{} \frac{1}{\Delta_i} + \frac{\Delta_i}{\mu_i^2} \lesssim \sum_{\substack{i\in [s]\\\Delta_i>0}}^{}\max\Big\{ \frac{1}{\Delta_i},\frac{\Delta_i}{\mu_i^2}\Big\}\ .
$$
Obviously, \algosparse~ is optimal for all those values of parameters. More importantly, its regret scales linearly with $s$ and is independent of the number of arms with non-positive means.

\bigskip

The minimax regret of \algosparse~ is necessarily of the same order of UCB, as they achieve the same regret when  $\frac{d-s}{\mu_1}- \sum_{\substack{i\in [s]\\\Delta_i>0}}^{}\frac{\Delta_i}{\mu_i^2}$ is arbitrarily close to 0. So we consider instead the minimax regret with respect to the distributions with a \emph{relative sparsity} level bounded away from 0, i.e., such that for some $\theta \in (0,1)$, $\mu_s \geq \theta \mu_1$. In particular, this yields that $\frac{d-s}{s} >\frac{1-\theta}{\theta^2}$ so that the strong sparsity assumption is satisfied.

Then, for this class of parameters, it is quite straightforward to get that the minimax regret of UCB scales as $\sqrt{dT\log(T)}\sqrt{1+\frac{\theta}{1-\theta}\frac{s}{d}}$ while the minimax regret of \algosparse~ increases as $\sqrt{sT\log(T)}\sqrt{\frac{(1-\theta)^2}{\theta^2}+1}$. As a consequence, for any fixed class of parameters, the dependency in the number of arms in the minimax regret shrinks from $\sqrt{d}$ to $\sqrt{s}$

\subsection{Variants \& small improvements}
Of course, the minimax regret of \algosparse~exhibits an extra $\sqrt{\log(T)}$ term, which is due to the fact that we used UCB as a basic algorithm. In the order hand, we could have used instead of UCB, any  variant such as UCB-2, improved-UCB, ETC, MOSS... In the same line of thoughts, the threshold of the force-log phase could also be updated to $2\sqrt{\frac{\log(T/N_i(t))}{N_i(t)}}$ so that the term $\sqrt{\log(T)}$ can be replaced by $\sqrt{\log(s)}$, which gives a regret scaling in 
\begin{align*}
\Reg(T) \lesssim  \left\{\begin{array}{ll}\sum_{\substack{i\in [s]\\\Delta_i>0}}^{} \frac{\log(T\Delta_i^2)}{\Delta_i} + \frac{\log(T\mu_i^2)\Delta_i}{\mu_i^2} & \text{ in the distribution dependent sense}\\
 \sqrt{sT}\sqrt{\frac{(1-\theta)^2}{\theta^2}+1}\sqrt{\log(s) + \log(\frac{(1-\theta)^2}{\theta^2}+1)}
 & \text{ in the minimax sense}\end{array}\right. \end{align*}

Similarly, under some additional assumptions on the probility distribution at stake, one might use KL-UCB instead of UCB to replace the dependency in $\Delta_i\Big(\frac{1}{\Delta^2_i} + \frac{\Delta_i}{\mu_i^2}\Big)$ into $\Delta_i\Big(\frac{1}{KL(\mu_i,\mu^*)} + \frac{\Delta_i}{KL(0,\mu_i)}\Big)$ when this makes sense\footnote{Notice that the sparse bandit problem is trivial with Bernoulli distributions.}.

\bigskip

Another way to slightly improve the guarantees of the algorithm is to change the round-robin phases into sampling phases in which arms are not selected uniformly at random but with probability depending on the past performances of the different arms, as in \cite{bubeck2013bounded}. Unfortunately, this does not improve the leading term (in $T$) of the regret, but merely the terms uniformly bounded (in $T$).
%
%
%\subsection{Weak sparsity}
%In the second case, it becomes
%\begin{equation}
%\label{cor:LB1}
%\liminf_{T\to \infty}\frac{R(T)}{\log(T)} \geq \sum_{i=1}^{k} 
%\frac{1}{2\Delta_i} + \sum_{i=k+1}^{s} 
%\frac{\mu_{k}^2}{\mu_i^2}\frac{\Delta_i}{2\Delta_{k}^2}
%+\frac{(d-s)}{2\mu_1}
%\left(1-\frac{\mu_{k}^2}{\Delta_{k}^2}\right).
%\end{equation}
%
%\todov{Est-ce qu'on peut switcher sur UCB dans ce cas ? Par exemple si $\bar{X}_i(t)$ est trop petit}
%%% Local Variables:
%%% mode: latex
%%% TeX-master: "main_colt"
%%% End:

\section{Experiments}
\label{sec:experiments}
%!TEX root = main_colt.tex
%

This section aims at experimentally validating the theoretical results
we obtained. We empirically compare the regret of \textsc{UCB} and 
\algosparse~for various levels of sparsity: we either fix $d$ and $s$ and allow $\mu_s/\Delta_s$ to vary or conversely fix the expected returns and allow $s/d$ to vary. According to the conclusions of Section~\ref{sec:optimality}, we observe that for a range of settings, \algosparse~does behave near-optimally in the long run, up to multiplicative constants. We also see that even when \algosparse~is not optimal, it is still almost always preferable to \ucb~as soon as there is some sparsity in the problem. Without loss of generality, experiments are performed on problems for which $\mu_1=0.9$ and for $2\leq i \leq s$, all $\mu_i$'s are equal to $\mu_s=\mu_1-\Delta_s$. 

\subsection{Varying $\mu_s$}

We fix $d=15$ and $s=7$ such that the limit between weak and strong sparsity as defined in Section~\ref{sec:optimality} is reached at $\mu_s=0.4$. We allow $\Delta_s$ to vary in $[0.1,0.7]$. We compare the behavior of $\algosparse$~ and $\ucb$~ for these 
bandit problems, and we also compute and display the lower bound of Corollary~\ref{cor:LB}, that is for the smallest class of sparse of problems containing ours -- for $\varepsilon=\mu_s$. 

On Figure~\ref{fig:comp_delta} we present the expected regret averaged
over Monte-Carlo 100 repetitions for each experiment. When the sparsity of the
problem is not strong, that is when $\Delta_s=0.7$, \ucb~ has a lower regret than \algosparse~for a long time but the asymptotic behavior of the latter tends to show that \ucb~ will eventually be worse in the long run. However, when the sparsity gets stronger, \algosparse~ is much closer to optimal than \ucb~ and reaches a much lower regret.

\begin{figure}
%	\centering
%	\begin{subfigure}
%		\centering
		\includegraphics[width=.25\linewidth]{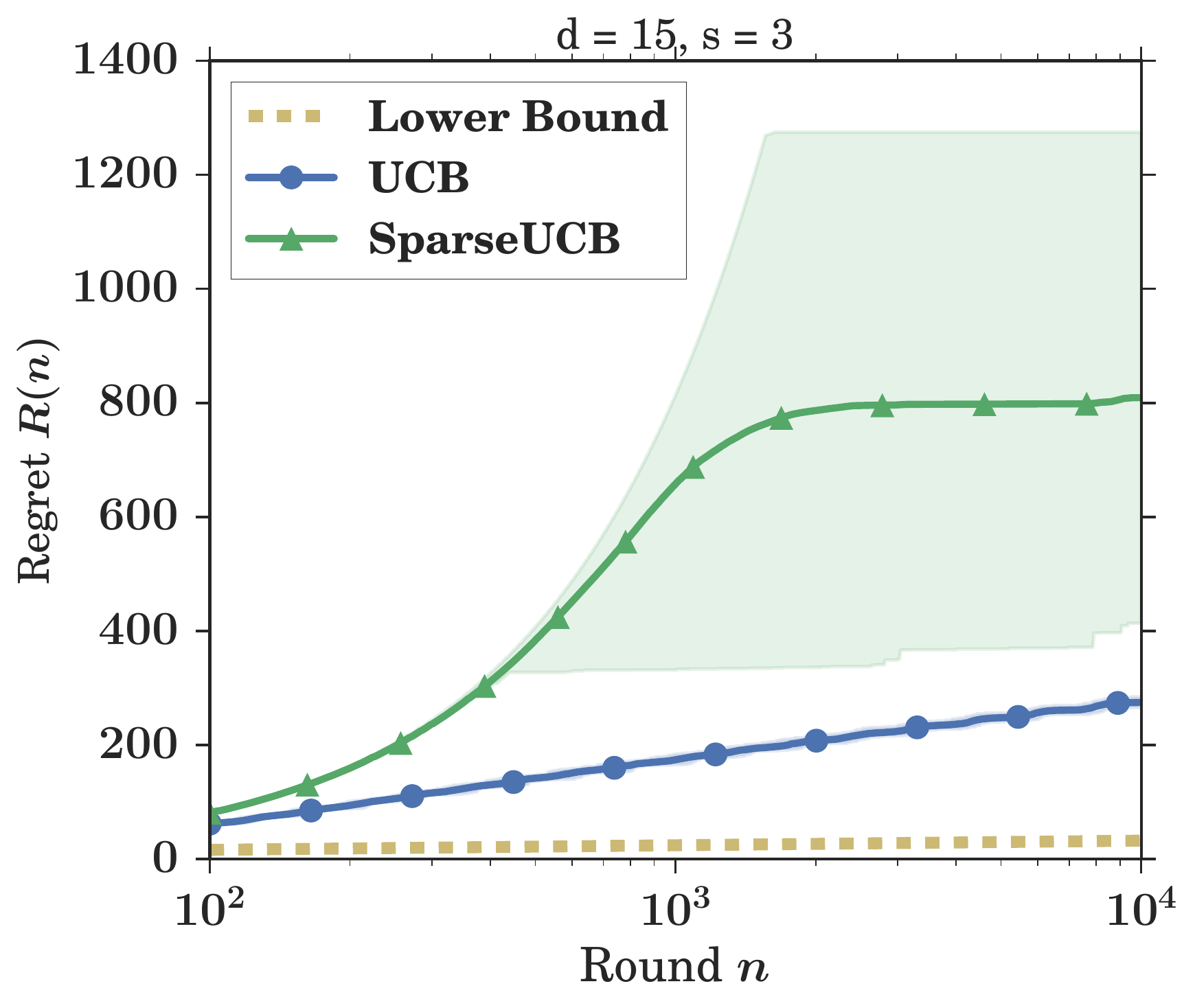}
%		\label{sub:comp_delta005}
		%\subcaption{$\Delta_s = 0.25$}
%	\end{subfigure}%
%	\begin{subfigure}
%		\centering
		\includegraphics[width=.25\linewidth]{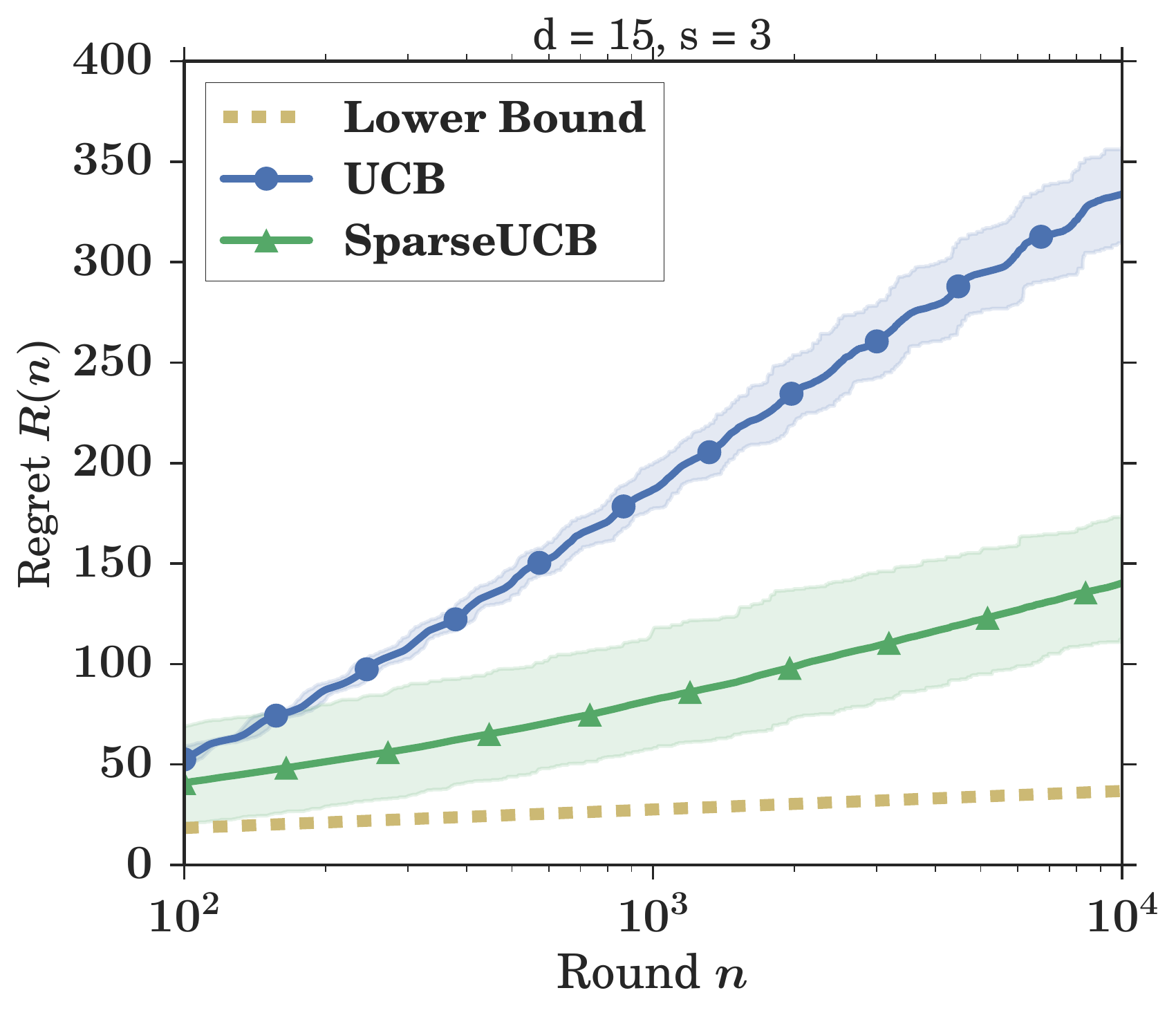}
%		\label{sub:comp_delta02}
		%\subcaption{$\Delta_s = 0.47$}
%	\end{subfigure}%
%	\begin{subfigure}
%		\centering
		\includegraphics[width=.25\linewidth]{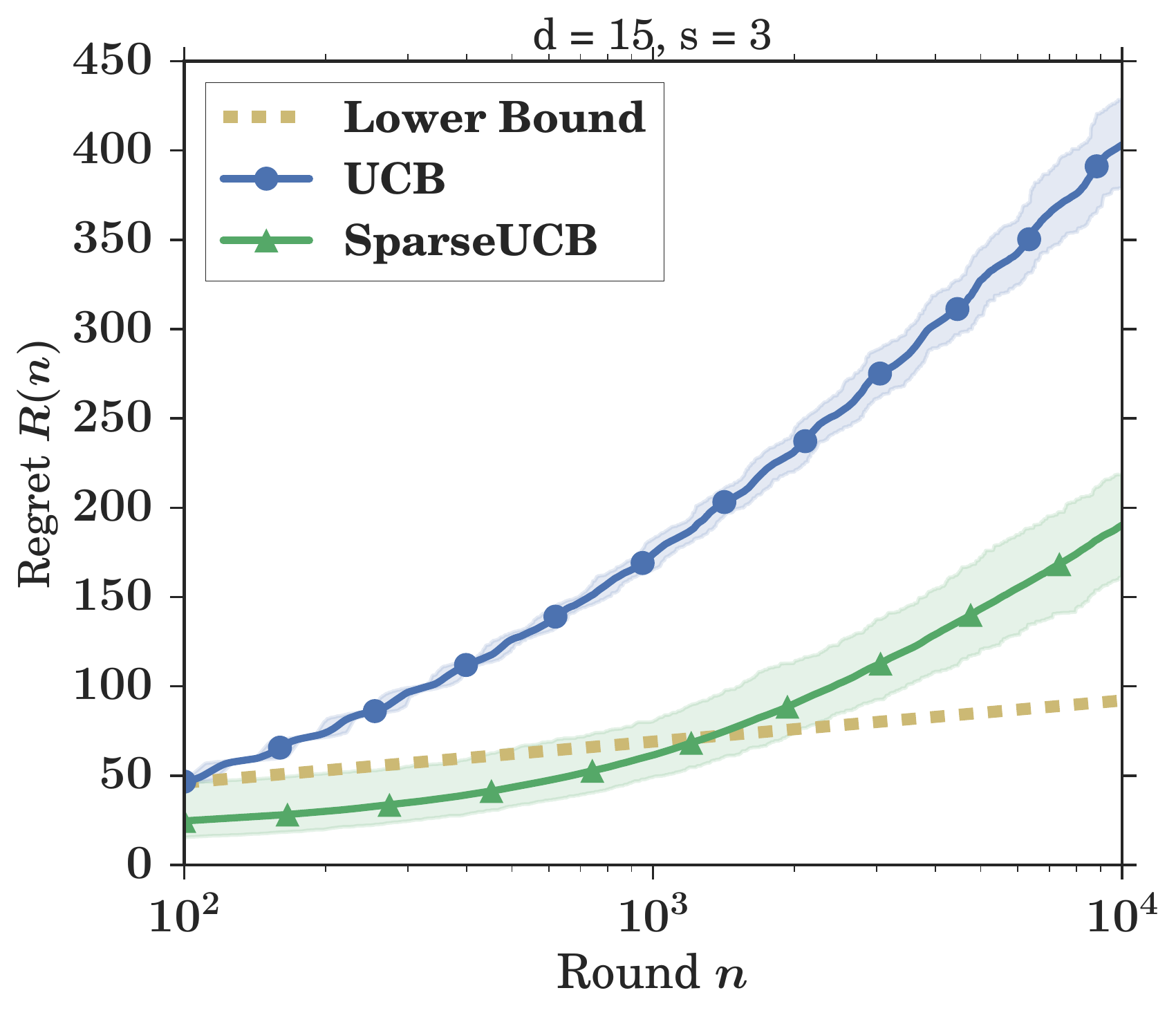}
%		\label{sub:comp_delta03}
		%\subcaption{$\Delta_s = 0.69$}
%	\end{subfigure}
	\caption{Cumulated regret of \ucb~ and \algosparse~ for $\mu^*=0.9$, $d=15$, $s=7$ and, from left to right, $\Delta_s=0.7,\,0.25, \,0.1$.} \label{fig:comp_delta}
\end{figure}

\subsection{Influence of the number of arms}

We now fix $d=15$, $\mu^*=0.9$ and $\Delta_s=0.3$ and we allow the number of effective arms $s$ to vary in $\{2,6,12\}$. Note that given the fixed parameters, the regime of weak sparsity defined in previous section only holds when $s>10$.

On Figure~\ref{fig:comp_arms} we present the expected regret averaged over 100 Monte-Carlo repetitions. Clearly, when $s=12$, we are in the weak sparsity regime and there is no real improvement brought by \algosparse~ as compared to the usual \ucb~policy. On the contrary, as $s$ gets smaller, \algosparse~gets closer and closer to optimal. 

\begin{figure}
%	\centering
%	\begin{subfigure}
%		\centering
		\includegraphics[width=.25\linewidth]{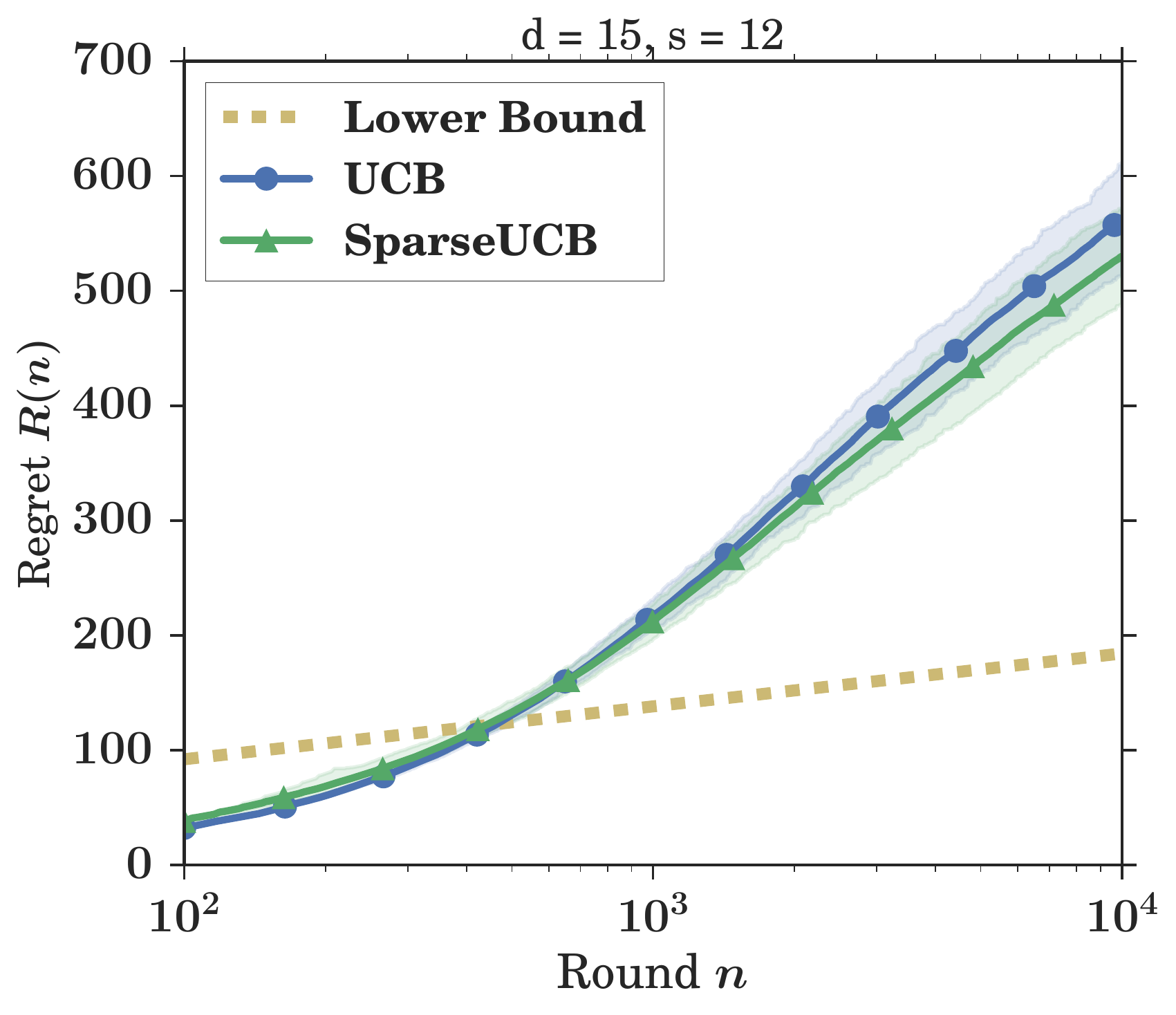}
%		\label{sub:comp_s12}
		%\subcaption{}
%	\end{subfigure}%
%	\begin{subfigure}
%		\centering
		\includegraphics[width=.25\linewidth]{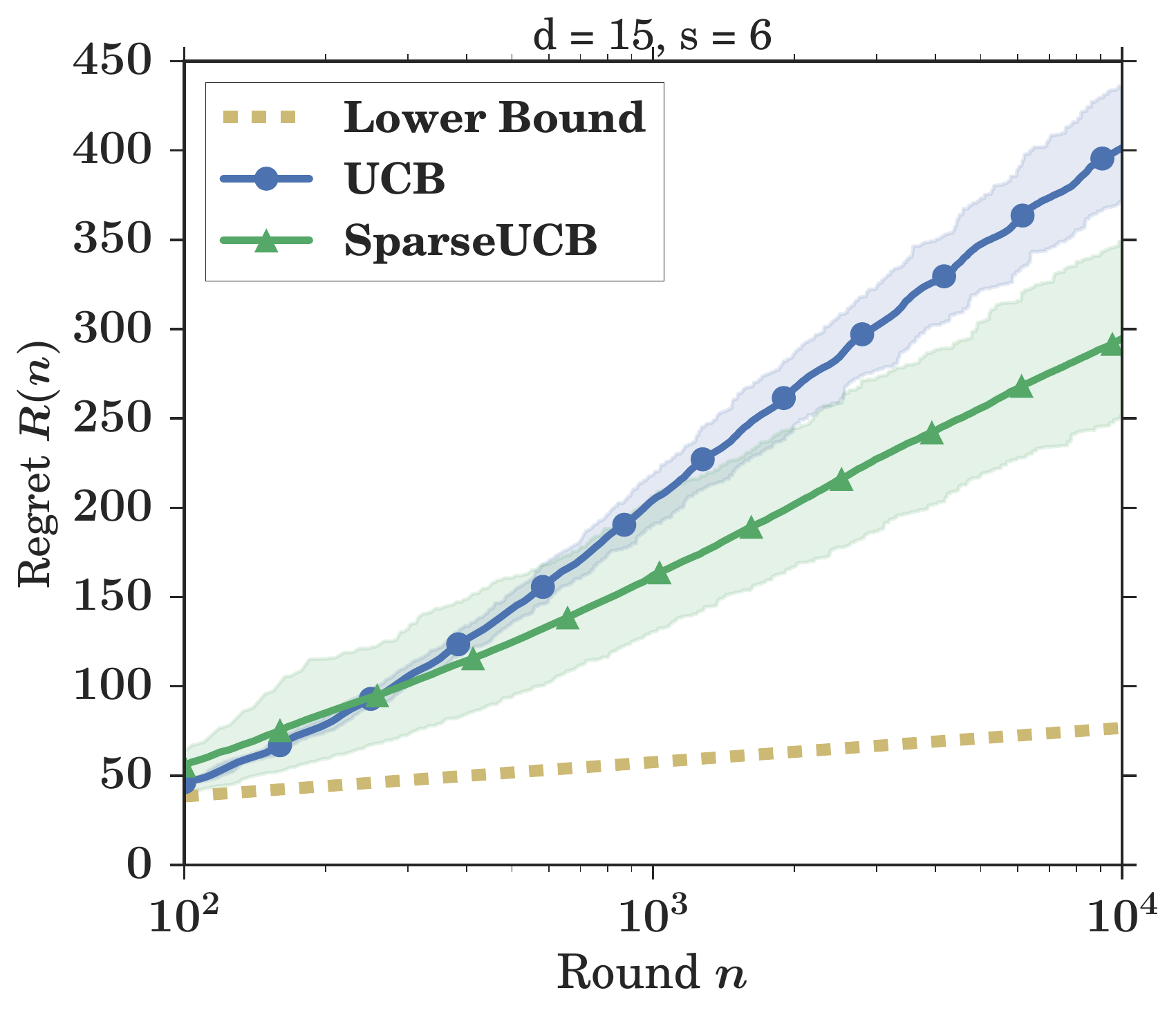}
%		\label{sub:comp_s6}
		%\subcaption{}
%	\end{subfigure}%
%	\begin{subfigure}
%		\centering
		\includegraphics[width=.25\linewidth]{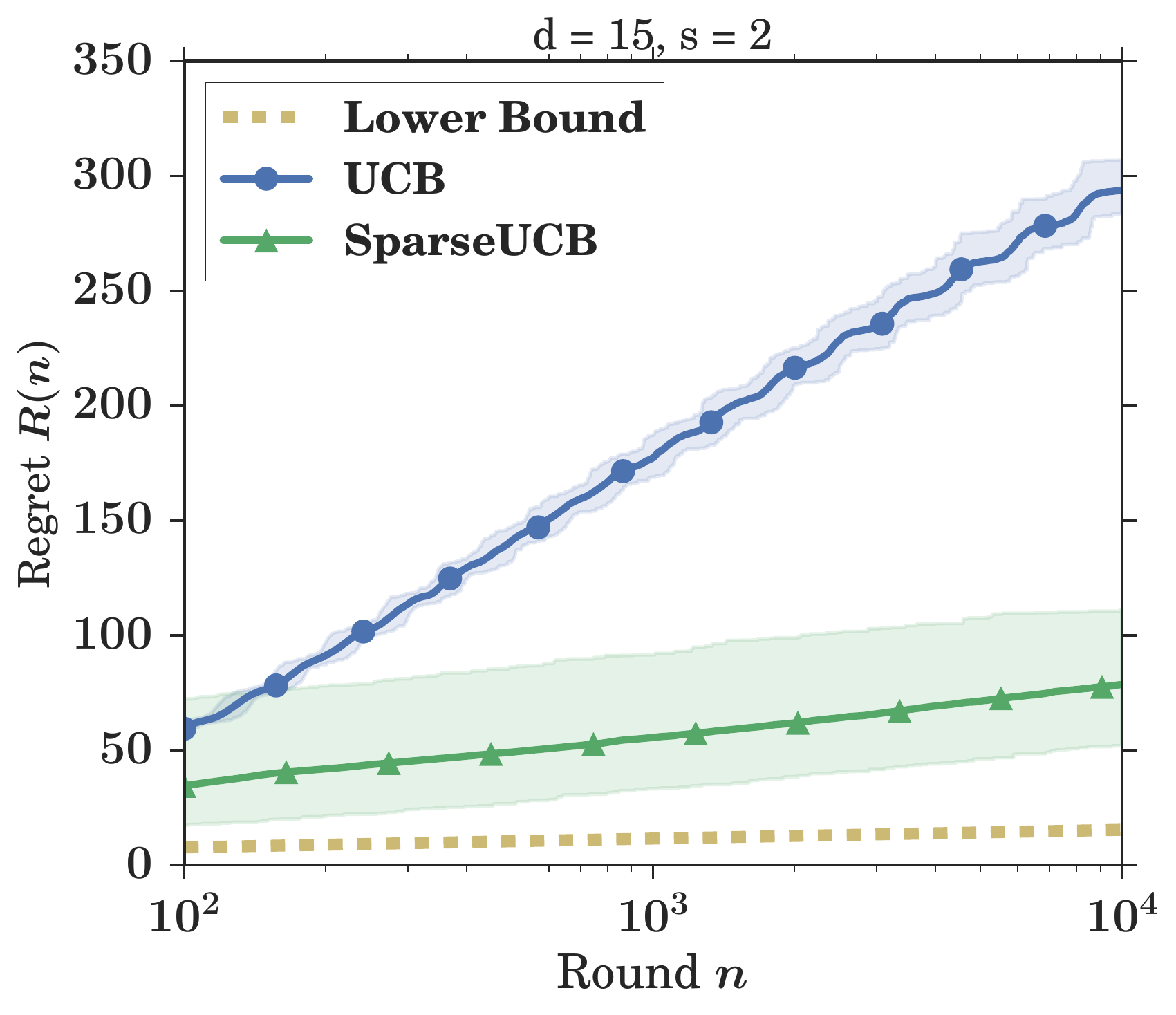}
%		\label{sub:comp_s2}
		%\subcaption{}
%	\end{subfigure}
	\caption{Cumulated regret of \ucb~ and \algosparse~ for $d=15$, $\mu=0.9$ and, from left to right, $s=12,6,2$.} \label{fig:comp_arms}
\end{figure}

%\begin{figure}
%\centering
%\includegraphics[width=0.7\linewidth]{./figures/figuretest}
%\caption{This is a very nice figure in log scale}
%\label{fig:figuretest}
%\end{figure}

%%% Local Variables:
%%% mode: latex
%%% TeX-master: "main_colt"
%%% End:

% \section{Combining UCB and SparseUSB (preliminary draft)}
% \input{combining.tex}

\section{Conclusions and Open Questions}
\label{sec:conclusion}
%!TEX root = main_colt.tex
%
We introduced a new variation of the celebrated stochastic multi-armed bandit problem that include an additional sparsity information on the expected return of the arms. 
We characterized the range of parameters that lead to interesting sparse problems and gave a lower bound on the regret that scales in $O(s\log(T))$ in the Strong Sparsity domain. 
% \textbf{However, it is still an open problem to prove a lower bound in $\sqrt{slog(T)}$ in the minimax sense}\todoc{really ?}\todov{Not sure  if we can adapt MOSS to our setting}...
We provide \algosparse~ that is a good alternative to the classical \ucb~ in the sparse bandit situation as it has both good theoretical guarantees and good empirical performances. However, we noticed in the experiments that for parameters lying in the Weak Sparsity domain, one would rather switch to the classical \ucb~ policy as the price of focusing on the $s$ best arms paid by \algosparse~ is too high as compared to the resulting improvement on the regret. 
Moreover, it appears in many real applications that the learner often knows the existence of $s$ without knowing its exact value and that leverages a new and unsolved stochastic sparse problem.

% Acknowledgments---Will not appear in anonymized version

%\bibliographystyle{unsrtnat}
\bibliography{refscolt}

\appendix

\section{End of Proof of Lower Bound}
\label{ap:LBproof}
\setcounter{theorem}{1}

Recall Theorem~\ref{cor:LB}:
\begin{theorem}
	For a Gaussian sparse bandit problem $\nu =
	(\nu_1,\ldots,\nu_s,\nu_{s+1},\ldots,\nu_d) \in
	\mathcal{S}(d,s)\in \mathbb{R}^d$, an asymptotic lower bound on 
	the regret is given by the solution to the following linear optimization 
	problem:
	\begin{eqnarray}
	&  & f(\mu)\geq \quad \inf_{c\succeq 0} c_i \Delta_i\\
	s.t. & \quad\forall i\in \{2,...,s\},
	& 2 c_i \Delta_i^2\geq 1;\\
	& \forall i\in \{2,\ldots,s\},\, \forall j\in \{s+1,...d\},  & 2 c_j 
	\mu_1^2+ 2 c_i \mu_i^2 \geq 1\\
	& \forall i\in \{1,\ldots,d\}, & c_i \geq 0
	\end{eqnarray}
	whose solution can be computed explicitly and gives the following 
	problem-dependent lower bound:
	\begin{itemize}
		\item If $\frac{d-s}{\mu_1}- 
			\sum_{\substack{i\in [s]\\\Delta_i>0}}^{}\frac{\Delta_i}{\mu_i^2} > 0$,
		\begin{equation}
		\label{cor:LB0}
		\liminf_{T\to \infty}\frac{\operatorname{Reg}(T)}{\log(T)} \geq 
			\sum_{\substack{i\in [s]\\\Delta_i>0}}^{}\max\Big\{ \frac{1}{2\Delta_i},\frac{\Delta_i}{2\mu_i^2}\Big\}
		\end{equation}
		\item otherwise, there exist $k\leq s$ such that $\frac{d-s}{\mu_1}- 
		\sum_{i=k}^{s}\frac{\Delta_i}{\mu_i^2} < 0$, and the lower bound is 
		\begin{equation}
		\label{cor:LB1}
		\liminf_{T\to \infty}\frac{\operatorname{Reg}(T)}{\log(T)} \geq \sum_{\substack{i\in [k]\\\Delta_i>0}}^{} 
		\frac{1}{2\Delta_i} + \sum_{i=k+1}^{s} 
		\frac{\mu_{k}^2}{\mu_i^2}\frac{\Delta_i}{2\Delta_{k}^2}
		+\frac{(d-s)}{2\mu_1}
		\left(1-\frac{\mu_{k}^2}{\Delta_{k}^2}\right).
		\end{equation}
	\end{itemize}

\end{theorem}

\begin{proof}
We now solve the following linear programming in 
	order to obtain the given explicit form for the Lower Bound. 	
	\begin{eqnarray*}
	&  & f(\mu)\geq \quad \inf_{c\succeq 0} c_i \Delta_i\\
	s.t. & \quad\forall i\in \{2,...,s\},
	& 2 c_i \Delta_i^2\geq 1;\\
	& \forall i\in \{2,\ldots,s\},\, \forall j\in \{s+1,...d\},  & 2 c_j 
	\mu_1^2+ 2 c_i \mu_i^2 \geq 1\\
	& \forall i\in \{1,\ldots,d\}, & c_i \geq 0
	\end{eqnarray*}
	
	First, remark that  for all the 
	best arms $i\in \{1,\ldots,s\}$, we must have $c_i\geq 1/2\Delta_i^2$ 
	so if $\mu_i^2/\Delta_i^2 \geq 1$, the $d-s$ corresponding constraints 
	on suboptimal $c_j,j>s$ are empty. We define 
	$S^*:=\left\{  i\in [s]\,\middle|\,\mu_i^2/\Delta_i^2 \geq 1 \right\}$. It remains $s-\vert S^*\vert$ 
	constrains on each $c_j$ for $j>s$:
	\[
	c_j \geq \max_{i\in [s]\setminus S^*} 
	\frac{1-2c_i\mu_i^2}{2\mu_1^2} =: 
	\frac{\lambda}{2 \mu_1^2}
	\]
	It remains to properly identify $\lambda$ as a function of the 
	parameters of the problem. Because of the first set of constraints, for all 
	$i\notin S^*$,
	\[
	c_i = \max \left\{\frac{1}{2\Delta_i^2}, 
	\frac{1-\lambda}{2\mu_i^2} \right\}
	\]
	For those coefficients $i\leq s$ such that 
	$c_i=\frac{1-\lambda}{2\mu_i^2}$, we have 
	\[
	\lambda \leq 
	1-\left(\frac{\mu_i}{\Delta_i}\right)^2:=\Theta_i \in [0,1]
	\]
	where the quantity $\Theta_i$ increases with $i$, i.e. the worse the arm 
	is, the bigger his $\Theta_i$. Let $k\notin S^*$ be the smaller index 
	such that
	\begin{equation}
	\label{eq:def_k}
	\Theta_{k-1}  < \lambda \leq \Theta_{k}.
	\end{equation} 
	Then, we set the values of the coefficients as 
	\[
		\begin{cases}
		c_i = 1/2\Delta_i^2 & i< k \\
		c_i = (1-\lambda)/2\mu_1^2 & i \geq k
		\end{cases}
	\]

	We can rewrite the optimization problem as a function of $\lambda > 
	\Theta_{k-1}$:
	\begin{align}
	\notag f(\mu) \geq &\sum_{i=1}^{k-1}\frac{1}{2\Delta_i} + \sum_{i=k}^{s} 
	\Delta_i\frac{1-\lambda}{2\mu_i^2} +(d-s)\frac{\lambda}{2\mu_1}\\
	& =\sum_{i=1}^{k-1}\frac{1}{2\Delta_i} + \sum_{i=k}^{s} 
	\frac{\Delta_i}{2\mu_i^2} 
	+\frac{\lambda}{2}\left(\frac{d-s}{\mu_1}-\sum_{i=k}^{s} 
	\frac{\Delta_i}{\mu_i^2} \right)	\label{eq:opt_lambda}
	\end{align}  	  	
	
	Now we must distinguish two cases depending on the sign of 
	\begin{equation}
		\label{eq:sign_check}
		\frac{d-s}{\mu_1}-\sum_{i=k}^{s} 
		\frac{\Delta_i}{\mu_i^2} 
	\end{equation}
	\paragraph{Strong sparsity.} If $\frac{d-s}{\mu_1}- 
	\sum_{\substack{i\in [s]\\\Delta_i\geqslant 0}}\frac{\Delta_i}{\mu_i^2} > 0$, then we must set the coefficients such that $\lambda$ reaches its lowest allowed value, which is $\lambda=0$. Hence, $c_i=\max\{ \frac{1}{2\Delta_i^2},\frac{1}{2\mu_i^2}\}$ for all $i\leq s$. The lower bound 
	is then 
	\[
f(\mu) = \sum_{\substack{i\in [s]\\\Delta_i>0}}\max\Big\{ \frac{1}{2\Delta_i},\frac{\Delta_i}{2\mu_i^2}\Big\}.	\]

	\paragraph{Weak sparsity} Otherwise, there exists $k\leq s$ such that the 
	expression of Eq.~\ref{eq:sign_check} is negative. Then, we have by 
	definition of $k$,
	\[
	\lambda = \Theta_{k}
	\]
	and, rearranging the terms of Eq.\eqref{eq:opt_lambda}, the Lower Bound 
	finally writes 
	\[
	\liminf_{T\to \infty}\frac{\operatorname{Reg}(T)}{\log(T)} \geq \sum_{i=1}^{k} 
	\frac{1}{2\Delta_i} + \sum_{i=k+1}^{s} 
	\frac{\mu_{k}^2}{\mu_i^2}\frac{\Delta_i}{2\Delta_{k}^2}
	+\frac{(d-s)}{2\mu_1}
	\left(1-\frac{\mu_{k}^2}{\Delta_{k}^2}\right).
	\]
	A special case of the above bound is when $k=s$. Then 
	\[
	\lambda = 1-\frac{\mu_s^2}{\Delta_s^2}.
	\]
	In that case, the lower bound is 
	\[
	f(\mu) \geq \sum_{i=1}^{s}\frac{1}{2\Delta_i} + 
	\sum_{i=s+1}^{d}\frac{\Delta_i (1-\mu_s^2/\Delta_s^2)}{2\mu_1^2} = 
	\sum_{i=1}^{d}\frac{1}{2\Delta_i} 
	-\frac{(d-s)}{2\mu_1}\frac{\mu_s^2}{\Delta_s^2}
	\]  
\end{proof}

\section{Generalization of Theorem~\ref{cor:LB}}
\label{ap:LB2}

The lower bound of Theorem~\ref{cor:LB} can be generalized to a wider class of problems including a sparsity information. 
We assume that we know $\epsilon > 0$ such that $\mu_{s+1} > \mu_s-\epsilon$.
A sparse bandit problem as defined in Section~\ref{sec:framework} is at least included in such class of problem for $\epsilon = \mu_s$. Introducing $\epsilon>0$ allows us to provide a result that applies to our problem as well as to similar ones such as the \emph{Stochastic Thresholded Bandit}
\footnote{This problem has not been studied yet in a regret minimization setting to our knowledge but its setting is close to ours.} 
for which one would assume that there exists a threshold $\mu_s\geq \tau>0$ such that the $s$ arms of interest have an expected return at of at least $\tau$. In that case, a change of variable $\epsilon \gets \mu_s-\tau$ in the following Theorem provides a lower bound on the regret of any uniformly efficient algorithm for that problem. We chose to introduce this wilder class of problems in order to provide a generic result and its associated proof technique, but we state the specific lower bound for our own problem in Theorem~\ref{cor:LB} below.

\begin{theorem}
	For a Gaussian sparse bandit problem $\nu =
	(\nu_1,\ldots,\nu_s,\nu_{s+1},\ldots,\nu_d) \in
	\mathcal{S}(d,s,\epsilon)\in \mathbb{R}^d$, an asymptotic lower bound on 
	the regret is given by 	\begin{itemize}
		\item If $\frac{d-s}{\mu_1}- 
			\sum_{\substack{i\in [d]\\\Delta_i>0}}\frac{\Delta_i}{(\mu_i-\mu_s+\epsilon)^2} > 0$,
			\[
			f(\mu) = \sum_{\substack{i\in [s]\\\Delta_i>0}}\max\Big\{ \frac{1}{2\Delta_i},\frac{\Delta_i}{2(\mu_i-\mu_s+\epsilon)^2}\Big\};
			\]
		\item otherwise, there exist $k\leq s$ such that $\frac{d-s}{\mu_1}- 
		\sum_{i=k}^{s}\frac{\Delta_i}{(\mu_i-\mu_s+\epsilon)^2} < 0$, and the lower bound is 
		\begin{equation}
		\label{eq:LB2}
		\liminf_{T\to \infty}\frac{\operatorname{Reg}(T)}{\log(T)} \geq  \sum_{i=1}^{k} 
	\frac{1}{2\Delta_i} + \sum_{i=k+1}^{s} 
	\frac{(\mu_{k}-\mu_s+\epsilon)^2}{(\mu_i-\mu_s+\epsilon)^2}\frac{\Delta_i}{2\Delta_{k}^2}
	+\frac{(d-s)}{2\mu_1}
	\left(1-\frac{(\mu_{k}-\mu_s+\epsilon)^2}{\Delta_{k}^2}\right)..
		\end{equation}
	\end{itemize}

\end{theorem}

\section{Analysis of the \algosparse~algorithm}
\label{sec:analys-algosp-algor}
We provide in this section the detailed statements and proofs
concerning the upper bound guaranteed by the \algosparse~algorithm.

\setcounter{theorem}{3}

\begin{theorem}
\label{thm:regret-bound-precise}
For $T\geqslant 1$, the \algosparse~algorithm guarantees:
\begin{align*}
\mathbb{E}\left[   \operatorname{Reg}(T) \right] &\leqslant 16\log(T)
  \sum_{\substack{i\in [s]\\\Delta_i>0}} \left(
  \frac{1}{\Delta_i}+\frac{\Delta_i}{\mu_i^2} \right)+\left(
  \sum_{i\in [d]}^{}\Delta_i \right)\left(
  1+3s+\sum_{j=1}^s\frac{1+4\log (16/\mu_j^2)}{\mu_j^2} \right)\\
&\qquad \qquad \qquad +\sum_{\substack{i\in [s]\\\Delta_i>0}}^{}\Delta_i\left( 3+\frac{8}{\mu_i^2}+\frac{8}{\Delta_i^2} \right)+\frac{\pi^2}{6}\sum_{i=s+1}^d\Delta_i +\frac{d\Delta_s\pi^2}{6}. 
\end{align*}
\end{theorem}

We gather without proof in the following lemma a few properties which are immediate from the definition of the algorithm.
\begin{lemma}
\label{lm:by-definition}
By definition of the algorithm, we have:
\begin{enumerate}[(i)]
\item\label{item:3} For all $t\geqslant d+1$ and $i\in [d]$,
  $N_i(t)\geqslant 1$ and the sets $\mathcal{J}(t)$ and $\mathcal{K}(t)$ are well-defined.
\item\label{item:1} For all $t\geqslant d+1$ and $i\in [d]$, if arm
  $i$ is pulled at time $t$ during a \emph{round-robin} phase, then the
  set $\mathcal{J}(t-i+1)$ contains less than $s$ arm. In other words,
  \[ R_i(t)\subset \left\{ \left| \mathcal{J}(t-i+1) \right| <s \right\}.  \]

\item\label{item:2} For all $t\geqslant 1$ and $i,k\in [d]$, arm
  $i$ is pulled at time $t$ during a \emph{round-robin} phase if, and
  only if arm $k$ is also pulled during a \emph{round-robin} phase at time $t-i+k$. In other words,
  \[ \left\{ I(t)=i,\ \omega(t)=\mathfrak{r} \right\}=  \left\{ I(t-i+k)=k,\ \omega(t-i+k)=\mathfrak{r} \right\}.\]

\item\label{item:log} For all $t\geqslant d+1$ and $i\in [d]$, if arm
  $i$ is pulled at time $t$ during a \emph{force-log} phase, it does
  not belong to $\mathcal{K}(t)$ i.e.\ its empirical mean at time $t$ is strictly
  below $2\sqrt{\log (t)/N_i(t)}$. In other words,
  \[ F_i(t)\subset \left\{ \overline{X}_i(N_i(t))<2\sqrt{\frac{\log(t)}{N_i(t)}},\ I(t)=i \right\}.  \]

\item\label{item:subset} For all $t\geqslant 1$ and $i\in [d]$, arm
  $i$ is pulled at time $t$ during a \emph{force-log} or a \emph{UCB}
  phase only if it belongs to the set $\mathcal{J}(t)$. In other words, $F_i(t)$, $U_i(t)$ and $V_i(t)$ are subsets of $A_i(t)$.

\item\label{item:wrong-ucb} For all $t\geqslant d+1$,
  if an arm is pulled at time $t$ during an \emph{UCB} phase while
  the best arm does not belong to the set $\mathcal{K}(t)$,
  necessarily, a bad arm $j\in \left\{ s+1,\dots,d \right\}$ belongs to $\mathcal{K}(t)$. In other words,
  \[ \bigsqcup_{i\in [d]}V_i(t)\subset  \bigcup_{j>s}^{}\left\{ \overline{X}_j(N_j(t))\geqslant 2\sqrt{\frac{\log(t)}{N_j(t)}} \right\}.  \]
\end{enumerate}
\end{lemma}

\begin{lemma}
  \label{lm:round-robin}
  For $i\in [d]$, the number of times arm $i$ is pulled, while the
  algorithm is performing a \emph{round-robin} phase, is bounded in expectation as:
  \[ \mathbb{E}\left[ \sum_{t=1}^{+\infty}\mathbbm{1}\left\{ R_i(t)
      \right\}\right]\leqslant 1+3s+ \sum_{j=1}^s\frac{1}{\mu_j^2}\left( 8+32\log \left(  \frac{16}{\mu_j^2} \right) \right).    \]
\end{lemma}
\begin{proof}
  By definition of the algorithm, arm $i$ is pulled exactly once
  during the first $d$ stages:
  \[ \sum_{t=1}^d\mathbbm{1}\left\{ R_i(t) \right\}=1. \]
  Let $t\geqslant d+1$. Using the definition of $R_i(t)$ and property (\ref{item:1}) from
 Lemma \ref{lm:by-definition}, we write
 \[ R_i(t)=\left\{ I(t)=i,\ \omega(t)=\mathfrak{r} \right\}\cap
   \left\{ \left| \mathcal{J}(t-i+1) \right| <s \right\}.   \]
 If $\left| \mathcal{J}(t-i+1) \right| < s$, necessarily, there
 exists $j\in [s]$ such that $j\not \in \mathcal{J}(t-i+1)$, in other
 words, such that:
 \[ \overline{X}_j(N_j(t-i+1))<2\sqrt{\frac{\log (N_j(t-i+1))}{N_j(t-i+1)}}. \]
Thus, we write:
 \[ R_i(t)\subset \bigcup_{j=1}^s\left\{ \overline{X}_j(N_j(t-i+1))<2\sqrt{\frac{\log (N_j(t-i+1))}{N_j(t-i+1)}},\ I(t)=i,\ \omega(t)=\mathfrak{r} \right\}.  \]
 Therefore,
\begin{align}
\sum_{t=d+1}^{+\infty}\mathbbm{1}\left\{ R_i(t) \right\}&\leqslant
                                                             \sum_{j=1}^s
                                                             \sum_{t=d+1}^{+\infty}
                                                             \mathbbm{1}\left\{
                                                             \overline{X}_j(N_j(t-i+1))<2\sqrt{\frac{\log (N_j(t-i+1))}{N_j(t-i+1)}},\
                                                             I(t)=i,\
                                                             \omega(t)=\mathfrak{r}
                                                             \right\}\nonumber\\
&=\sum_{j=1}^s \sum_{\substack{d<
  t<+\infty\\I(t)=i\\\omega(t)=\mathfrak{r}}}^{} \mathbbm{1}\left\{ \overline{X}_j(N_j(t-i+1))<2\sqrt{\frac{\log( N_j(t-i+1))}{N_j(t-i+1)}} \right\}.
\label{eq:1}
\end{align}
For a given arm $j\in [s]$, the quantity $N_j(t-i+1)$ in the above
last sum is (strictly) increasing. Indeed, let $t<t'$ such that $I(t)=I(t')=i$
and $\omega(t)=\omega(t')=\mathfrak{r}$. As a consequence of property
(\ref{item:2}) from Lemma \ref{lm:by-definition}, we have
\begin{enumerate}[(i)]
\item\label{item:8} $I(t-i+k)\neq 1$ for $k\in \left\{ 2,\dots,d \right\}$;
\item\label{item:9} $I(t'-i+1)=1$.
\item\label{item:10} $I(t-i+j)=j$;
\end{enumerate}
The above properties (\ref{item:8}) and (\ref{item:9}) imply
$t-i+d\leqslant t'-i$, which in turn, together with property
(\ref{item:10}), gives that arm $j$ is pulled at least once between
time $t-i+1$ and $t'-i$ (at time $t-i+j$). Therefore,
$N_j(t-i+1)<N_j(t'-i+1)$. Therefore, the last sum in
Equation~\eqref{eq:1} can be bounded, with a change of variable, as
\[ \sum_{\substack{d<
  t<+\infty\\I(t)=i\\\omega(t)=\mathfrak{r}}}^{} \mathbbm{1}\left\{ \overline{X}_j(N_j(t-i+1))<2\sqrt{\frac{\log (N_j(t-i+1))}{N_j(t-i+1)}} \right\}\leqslant  \sum_{u=1}^{+\infty}\mathbbm{1}\left\{ \overline{X}_j(u)<2\sqrt{\frac{\log (u)}{u}} \right\}. \]
Going back to Equation~\eqref{eq:1}, we can now bound the expectation of the number of times arm $i$ was
pulled after time $d$ as follows:
\begin{align*}
\mathbb{E}\left[ \sum_{t=d+1}^{+\infty}\mathbbm{1}\left\{ R_i(t)
  \right\}  \right] &\leqslant \sum_{j=1}^s
                      \sum_{u=1}^{+\infty}\mathbb{P}\left[
                      \overline{X}_j(u)<2\sqrt{\frac{\log (u)}{u}}
                      \right]\\
&\leqslant \sum_{j=1}^s \sum_{u=1}^{+\infty}\mathbb{P}\left[ \overline{X}_j(u)-\mu_j<2\sqrt{\frac{\log (u)}{u}}-\mu_j \right] .
\end{align*}
One can easily check that
\[ u\geqslant 3+\frac{32}{\mu_j^2}\log  \left(  \frac{16}{\mu_j^2} \right) \quad \text{implies}\quad 2\sqrt{\frac{\log (u)}{u}}-\mu_j\leqslant -\frac{\mu_j}{2}. \]
Therefore, we set $u_j:=3+\lceil (32/\mu_j^2)\log (16/\mu_j^2)\rceil$ and write:
\begin{align*}
  \mathbb{E}\left[ \sum_{t=d+1}^{+\infty}\mathbbm{1}\left\{ R_i(t)
  \right\}  \right]&\leqslant \sum_{j=1}^s
                     \sum_{u=1}^{+\infty}\mathbb{P}\left[
                     \overline{X}_j(u)-\mu_j<2\sqrt{\frac{\log (u)}{u}}-\mu_j \right]\\
&\leqslant   \sum_{j=1}^s\left(
  3+\frac{32}{\mu_j^2}\log \left(   \frac{16}{\mu_j^2} \right) +\sum_{u=u_j}^{+\infty}\mathbb{P}\left[
  \overline{X}_j(u)-\mu_j<-\frac{\mu_j}{2} \right]  \right)\\
&\leqslant \sum_{j=1}^s\left(
  3+\frac{32}{\mu_j^2}\log \left(   \frac{16}{\mu_j^2}\right)+\sum_{u=u_j}^{+\infty} e^{-u\mu_j^2/8} \right)\\
&\leqslant 3s+ \sum_{j=1}^s\frac{1}{\mu_j^2}\left( 8+32\log \left(  \frac{16}{\mu_j^2} \right) \right) .
\end{align*}
\end{proof}

\begin{lemma}
  For $i\in [s]$ and $T\geqslant d+1$, the number of times arm $i$
  is pulled up to time $T$ during 
  \emph{force-log} phases, is bounded in expectation as:
  \[ \mathbb{E}\left[ \sum_{t=1}^{T} \mathbbm{1}\left\{ F_i(t)
      \right\}  \right]\leqslant \frac{16\log (T)+8}{\mu_i^2}.  \]
\end{lemma}
\begin{proof}
  Let $i\in [s]$ and $T\geqslant d+1$. By definition of the algorithm,
  $\mathbbm{1}\left\{ F_i(t) \right\}=0$ for $t\leqslant d$. Using property (\ref{item:log}) from Lemma~\ref{lm:by-definition}, we write
\begin{align*}
\sum_{t=d+1}^{T}\mathbbm{1}\left\{ F_i(t) \right\}&\leqslant
                                                    \sum_{t=d+1}^{T}\mathbbm{1}\left\{
                                                    \overline{X}_i(N_i(t))<2\sqrt{\frac{\log (t)}{N_i(t)}},\
                                                    I(t)=i \right\}\\
&\leqslant \sum_{\substack{d+1\leqslant t<+\infty\\I(t)=i}}^{}\mathbbm{1}\left\{
  \overline{X}_i(N_i(t))<2\sqrt{\frac{\log (T)}{N_i(t)}} \right\}.
\end{align*}
The quantity $N_i(t)$ being increasing in the above sum, with a change of
variable, we write:
\[ \sum_{t=d+1}^T\mathbbm{1}\left\{ F_i(t) \right\}\leqslant \sum_{u=1}^{+\infty}\mathbbm{1}\left\{ \overline{X}_i(u)<2\sqrt{\frac{\log (T)}{u}} \right\}.   \]
We now take the expectation:
\begin{align*}
\mathbb{E}\left[  \sum_{t=d+1}^T\mathbbm{1}\left\{ F_i(t) \right\}  \right]\leqslant \sum_{u=1}^{+\infty}\mathbb{P}\left[ \overline{X}_i(u)<2\sqrt{\frac{\log (T)}{u}} \right] =\sum_{u=1}^{+\infty}\mathbb{P}\left[ \overline{X}_i(u)-\mu_i<2\sqrt{\frac{\log (T)}{u}}-\mu_i \right].
\end{align*}
We consider $u_0:=\lceil 16 \log(T)/\mu_i^2\rceil$ which gives that $2\sqrt{\log (T)/u}-\mu_i\leqslant -\mu_i/2$ as soon as $u\geqslant u_0$. Therefore,
\begin{align*}
\sum_{u=1}^{+\infty}\mathbb{P}\left[
  \overline{X}_i(u)-\mu_i<2\sqrt{\frac{\log(T)}{u}}-\mu_i
  \right]&\leqslant \frac{16\log(T)}{\mu_i^2}
           +\sum_{u=u_0}^{+\infty}\mathbb{P}\left[
           \overline{X}_i(u)-\mu_i<-\frac{\mu_i}{2} \right]\\
&\leqslant \frac{16\log (T)}{\mu_i^2}+\sum_{u=u_0}^{+\infty}e^{-u\mu_i^2/8}\\
&\leqslant \frac{16\log(T)}{\mu_i^2}+\frac{8}{\mu_i^2},
\end{align*}
hence the result.
\end{proof}

\begin{lemma}
  For $i\in [s]$ such that $\Delta_i>0$ and $T\geqslant 1$, we have
  \[ \mathbb{E}\left[  \sum_{t=1}^T\mathbbm{1}\left\{ U_i(t) \right\}  \right]\leqslant \frac{16\log(T)+8}{\Delta_i^2}+3.  \]
\end{lemma}
\begin{proof}
Let $i\in [s]$ such that $\Delta_i>0$ and $t\geqslant 1$, and assume that:
\[ i\in \argmax_{j\in
  \mathcal{K}(t)}\left\{ \overline{X}_j(N_j(t))+ 2\sqrt{\frac{\log (t)}{N_j(t)}}
\right\}\quad \text{and}\quad 1\in \mathcal{K}(t). \]
In particular, we have:
\[ \overline{X}_i(N_i(t))+2\sqrt{\frac{\log (t)}{N_i(t)}}\geqslant \overline{X}_1(N_1(t))+2\sqrt{\frac{\log (t)}{N_1(t)}}. \]
Using the definition of $\Delta_i$, the above inequality can be equivalently written:
\[ \overline{X}_i(N_i(t))-\mu_i\geqslant \frac{\Delta_i}{2}+\left( \frac{\Delta_i}{2}-2\sqrt{\frac{\log (t)}{N_i(t)}} \right)+\left( \overline{X}_1(N_1(t))-\mu_1+2\sqrt{\frac{\log (t)}{N_1(t)}} \right).   \]
We consider $\tau_i:=\lceil 16\log(t)/\Delta_i^2 \rceil$ and we can see that as soon as $N_i(t)\geqslant \tau_i$, we have:
\[ \frac{\Delta_i}{2}-2\sqrt{\frac{\log (t)}{N_i(t)}}\geqslant 0. \]
When this is the case, we either have
\[ \overline{X}_i(N_i(t))-\mu_i\geqslant \frac{\Delta_i}{2}\quad
  \text{or}\quad \overline{X}_1(N_1(t))-\mu_1\leqslant -2\sqrt{\frac{\log (t)}{N_1(t)}}.  \]
With the above in mind, we can write
\begin{align*}
\mathbb{E}\left[   \sum_{t=1}^T\mathbbm{1}\left\{ U_i(t)
  \right\}\right]&\leqslant \tau_i+\mathbb{E}\left[
                   \sum_{\substack{1\leqslant t\leqslant
                   n\\N_i(t)\geqslant \tau_i}}^{}\left(
  \mathbbm{1}\left\{ \overline{X}_i(N_i(t))-\mu_i\geqslant
  \frac{\Delta_i}{2} \right\}\right.\right. \\
&\qquad \qquad \qquad \qquad\qquad  \left. {}\left. {}+\mathbbm{1}\left\{
  \overline{X}_1(N_1(t))-\mu_1\leqslant -2\sqrt{\frac{\log (t)}{N_1(t)}} \right\}
  \right)\right]\\
&\leqslant \tau_i+\sum_{u=\tau_i}^{+\infty}\mathbb{P}\left[ \overline{X}_i(u)-\mu_i\geqslant \frac{\Delta_i}{2} \right]+\sum_{t=1}^{+\infty}\mathbb{P}\left[ \overline{X}_1(N_1(t))-\mu_1\leqslant -2\sqrt{\frac{\log (t)}{N_1(t)}} \right].
\end{align*}
The arms being subgaussian, we bound the above probabilities as
follows:
\begin{align*}
\mathbb{E}\left[ \sum_{t=1}^T\mathbbm{1}\left\{ U_i(t) \right\}
  \right]    &\leqslant
               \tau_i+\sum_{u=\tau_i}^{+\infty}e^{-u\Delta_i^2/8}+\sum_{t=1}^T\frac{1}{t^2}\leqslant
               1+\frac{16}{\Delta_i^2}\log(T)+\frac{8}{\Delta_i^2}+\frac{\pi^2}{6}\\
&\leqslant \frac{16\log(T)+8}{\Delta_i^2}+3.
\end{align*}

\end{proof}

\begin{lemma}
  For $T\geqslant 1$, we have
  \[ \sum_{i\in [s]}^{}\Delta_i\mathbb{E}\left[ \sum_{t=1}^{T}\mathbbm{1}\left\{ V_i(t) \right\}  \right]\leqslant \frac{d\Delta_s\pi^2}{6}.  \]
\end{lemma}
\begin{proof}
  Using the fact that $\Delta_i\leqslant \Delta_s$ for all $i\in [s]$,
  \[ \sum_{i\in [s]}^{}\Delta_i\mathbb{E}\left[
      \sum_{t=1}^{T}\mathbbm{1}\left\{ V_i(t) \right\}
    \right]\leqslant \Delta_s\sum_{t=1}^T \mathbb{E}\left[ \sum_{i\in [s]}^{}\mathbbm{1}\left\{ V_i(t) \right\}  \right].   \]
  Using property (\ref{item:wrong-ucb}) from Lemma
  \ref{lm:by-definition}, we bound the above expectation as follows:

  \[\mathbb{E}\left[    \sum_{i\in [s]}^{}\mathbbm{1}\left\{ V_i(t) \right\}\right]\leqslant \sum_{j>s}^{}\mathbb{P}\left[ \overline{X}_j(N_j(t))\geqslant 2\sqrt{\frac{\log(t)}{N_j(t)}} \right].     \]
Using the fact the arms are subgaussian and that $\mu_j\leqslant 0$ (for $j>s$), we bound the above probability as:
\begin{align*}
\mathbb{P}\left[ \overline{X}_j(N_j(t))\geqslant 2\sqrt{\frac{\log (t)}{N_j(t)}} \right]=\mathbb{P}\left[
  \overline{X}_j(N_j(t))-\mu_j\geqslant 2\sqrt{\frac{\log(t)}{N_j(t)}}
  \right]\leqslant e^{-2\log(t)}=t^{-2}.
\end{align*}
The result follows.
\end{proof}

\begin{lemma}
  For $i\in \left\{ s+1,\dots,d \right\}$, we have
  \[ \mathbb{E}\left[ \sum_{t=1}^{+\infty}\mathbbm{1}\left\{ A_i(t) \right\}  \right]\leqslant \frac{\pi^2}{6}.  \]
\end{lemma}
\begin{proof}
  Let $i\in \left\{ s+1,\dots,d \right\}$. By definition of the algorithm, $\mathbbm{1}\left\{
    A_i(t) \right\}=0$ for $t\leqslant d$. Using the definition of $A_i(t)$, we write
  \[ \sum_{t=1}^{+\infty}\mathbbm{1}\left\{ A_i(t) \right\}= \sum_{t=d+1}^{+\infty}\mathbbm{1}\left\{ A_i(t) \right\}=\sum_{\substack{d<t<+\infty\\I(t)=i}}^{}\mathbbm{1}\left\{ \overline{X}_i(N_i(t))\geqslant 2\sqrt{\frac{\log(N_i(t))}{N_i(t)}} \right\}.    \]
In the last sum above, the quantity $N_i(t)$ is (strictly)
incrasing. Using a change of variable, and taking the expectation, we get:
\[ \mathbb{E}\left[   \sum_{t=d+1}^{+\infty}\mathbbm{1}\left\{
      A_i(t) \right\}\right]\leqslant \mathbb{E}\left[
    \sum_{u=1}^{+\infty} \mathbbm{1}\left\{
      \overline{X}_i(u)\geqslant 2\sqrt{\frac{\log(u)}{u}} \right\}
  \right]=\sum_{u=1}^{+\infty}\mathbb{P}\left[ \overline{X}_i(u)\geqslant 2\sqrt{\frac{\log(u)}{u}} \right].   \]
We now use the assumption that the arms have subgaussian laws and
that $\mu_{i}\leqslant 0$ to bound the above probability as:
\[ \mathbb{P}\left[ \overline{X}_i(u)\geqslant 2\sqrt{\frac{\log(u)}{u}} \right]\leqslant \mathbb{P}\left[ \overline{X}_i(u)-\mu_i\geqslant 2\sqrt{\frac{\log(u)}{u}} \right]\leqslant e^{-2\log(u)}=u^{-2}.   \]
The result follows.
\end{proof}

\end{document}